\documentclass[times,sort&compress,3p]{elsarticle}
\journal{Journal of Multivariate Analysis}
\usepackage[labelfont=bf]{caption}

\usepackage{amsmath,amsfonts,amssymb,amsthm,booktabs,color,epsfig,graphicx,hyperref,url}
\usepackage{enumitem,float,multirow,thmtools}

\theoremstyle{plain}% Theorem-like structures provided by amsthm.sty
\newtheorem{theorem}{Theorem}

\newtheorem{lemma}{Lemma}

\theoremstyle{definition}

\newtheorem{Remark}{Remark}
\newtheorem{Example}{Example}

\newtheoremstyle{Assumption}%                % Name
  {\topsep}% space above
  {\topsep}% space below
  {}% body font
  {}% indent amount
  {}% theorem head font
  {.}% punctuation after theorem head
  {.5em}% space after theorem head
  {\thmname{#1}\thmnumber{ #2}\thmnote{ (#3)}}% theorem head spec                
\theoremstyle{Assumption}
\newtheorem{Assumption}{Assumption}
\renewcommand{\theAssumption}{A\arabic{Assumption}}

\DeclareMathOperator*{\argmin}{argmin}

\numberwithin{equation}{section}

\def \mcF{\mathcal{F}}

\def \ev{\textrm{E}}
\def \pr{\textrm{Pr}}
\def \p{p}

\def \cid{\xrightarrow[\text{}]{\text{$\mathbb{L}$}}}
\def \cip{\xrightarrow[\text{}]{\text{$\mathbb{P}$}}}
\newcommand{\vo}{\vec{o}\@ifnextchar{^}{\,}{}}

\renewcommand\thmcontinues[1]{Continued}
\begin{document}

\begin{frontmatter}

\title{Online statistical inference for parameters estimation with linear-equality constraints}

\author[1]{Ruiqi Liu \corref{mycorrespondingauthor}}
\author[2]{Mingao Yuan}
\author[3]{Zuofeng Shang}

\address[1]{Department of Mathematics and Statistics, Texas Tech University, TX 79409, USA.}
\address[2]{Department of Statistics,   North Dakota State University, ND 58108, USA.}
\address[3]{Department of Mathematical Sciences, New Jersey Institute of Technology, NJ 07102, USA.}
\cortext[mycorrespondingauthor]{Corresponding author: \url{ruiqliu@ttu.edu}}

\begin{abstract}
Stochastic gradient descent (SGD) and projected stochastic gradient descent (PSGD) are scalable algorithms to compute model parameters in unconstrained and constrained optimization problems. In comparison with SGD, PSGD forces its iterative values into the constrained parameter space via projection.  From a statistical point of view, this paper studies the limiting distribution of PSGD-based estimate when the true parameters satisfy some linear-equality constraints.  Our theoretical findings reveal the role of projection played in the uncertainty of the PSGD-based estimate.  As a byproduct, we propose an online hypothesis testing procedure to test the linear-equality constraints. Simulation studies on synthetic data and an application to a real-world dataset confirm our theory.
\end{abstract}


\end{frontmatter}

\section{Introduction\label{sec:1}}
With the rapid increase in availability of data in the past two decades or so, many classical optimization methods for statistical problems such as gradient descent, expectation-maximization or Fisher scoring cannot be applied in the presence of large datasets, or when the observations are collected one-by-one in an online fashion \cite{su2018uncertainty, fang2018online}. To overcome the difficulty in the era of big data, a computationally scalable algorithm called stochastic gradient descent (SGD)  proposed in the seminal work \cite{robbins1951stochastic} has been widely applied and achieved great  success \cite{bottou1991stochastic, zhang2004linearprediction, rainer2011matrix}. In comparison with classical optimization methods, one appealing feature of SGD is that the algorithm only requires accessing a single observation during each iteration, which makes it scale well with big data and computationally feasible with streaming data.

Due to the success of SGD, the studies of its theoretical properties have drawn a great deal of attention. The theoretical analysis of SGD can be categorized into two directions based on different research interests. The first direction is about the convergence rate. Existing literature shows that SGD algorithm can achieve a (in terms of regret)  $O(1/T)$ convergence rate for strongly convex objective functions (e.g., see \cite{bottou2018optimization, gower2019sgd}), and a  $O(1/\sqrt{T})$ rate for general convex cases \cite{nemirovski2009robust}, where $T$ is the number of iterations.  The second direction focuses on applying SGD to statistical inference. It was proved that the SGD estimate is asymptotic normal (e.g., see \cite{pelletier2000asymptotic}) under suitable conditions. However, unlike classical parameter estimates, the SGD estimate may not be root-$T$ consistent, and its convergence rate depends on the learning rate. To improve the convergence rate,  \cite{polyak1990new} and \cite{ruppert1988efficient} independently proposed the averaged stochastic gradient descent (ASGD) estimate, which was obtained by averaging the updated values in all iterations. They showed that the ASGD estimate is root-$T$ consistent, while its asymptotic normality was proved by \cite{polyak1992acceleration}. Following \cite{polyak1992acceleration}, there is a vast amount of work related to conducting statistical inference based on ASGD estimates. For example, \cite{su2018uncertainty} proposed a hierarchical incremental gradient descent (HIGrad) procedure to  construct the confidence interval for the unknown parameters. In comparison with ASGD estimate, the flexible structure makes HiGrad easier to parallelize. In  \cite{fang2018online}, the authors developed an online bootstrap algorithm to construct  the confidence interval, which is still applicable when  there is no explicit formula for the covariance matrix of the ASGD estimate.  Recently, \cite{chen2020statistical} proposed a plug-in estimate and a batch-means estimate for the asymptotic covariance matrix. With strong convexity assumption on the objective function, they proved the convergence rate of the estimates. 

When there are constraints imposed on the parameters, the SGD algorithm is often combined with projection, which forces the iterated values into the constrained parameter space. The convergence rate of this projected stochastic gradient descent (PSGD) is also well studied (e.g., see \cite{nemirovski2009robust}), which is proved to be the same as that of SGD. In the view of statistical inference, \cite{jerome2005central} studied the asymptotic distribution of PSGD estimate when the model parameters are in the interior of the constrained parameter space. It was proved that  the projection operation only happens  a finite number of times almost surely. As a consequence, the limiting distribution of PSGD estimate is exactly the same as that of SGD estimate. Recently,  \cite{godichon2017averaged} studied the limiting distribution of averaged projected stochastic gradient descent (APSGD) estimate, which is the averaged version of PSGD. When the  model parameters are in the interior of the constrained parameter space, APSGD and ASGD estimates have the same limiting distribution.

This paper aims to quantify the uncertainty in APSGD estimates when the model parameters satisfy some linear-equality constraints. Compared to the existing literature, a significant difference of our model  is that the model parameters are not in the interior of the constrained parameter space. Therefore, the projection operation will take place during every iteration, and the limiting distribution of the APSGD estimate turns out to be a degenerate multivariate normal distribution. The contribution of current work is threefold:
\begin{enumerate}[label={\normalfont(\roman*)}, wide=0pt]
\item We derive the limiting distribution of the APSGD estimate, which is proved to be at least as  efficient as ASGD estimate under mild conditions. 
\item An online specification test for the linear-equality constraints is proposed based on the difference between APSGD and ASGD estimates.
\item  Our findings reveal that, when the true parameters are not in the interior of the parameter space, the APSGD and ASGD estimates could have different limiting distributions.
\end{enumerate}

This paper is organized as follows. In Section \ref{section:model}, we mathematically formulate the parameters estimation problem with linear-equality constraints. Section \ref{section:PSGD} proposes the APSGD estimate and studies its asymptotic properties. 
An online specification test  is proposed in Section \ref{section:test}. All the mathematical proofs are deferred to the appendix. A set of Monte Carlo simulations to investigate the finite sample performance of the proposed methods and an application to a real-world dataset are provided in a supplementary material.

\section{Problem formulation}\label{section:model}
We consider the problem to conduct statistical inference about the model parameter
\begin{eqnarray}
\theta^*=\argmin_{\theta \in \mathbb{R}^p}\big\{L(\theta):=\ev[l(\theta, Z)]\big\},\label{eq:model}
\end{eqnarray}
where  $l(\theta, Z)$ is the loss function, and $Z$ is a single copy drawn from an unknown distribution $F_{\theta^*}$. Moreover, we assume that additional information  about the truth $\theta^*$ is available:
\begin{eqnarray}
B\theta^*=b,\label{eq:linear:constraint}
\end{eqnarray}
where $B$ and $b$ are some prespecified matrix and vector with  comfortable dimensions. The loss function specified by (\ref{eq:model}) is quite general and covers many  popular statistical models, which are illustrated by the following examples.

\begin{Example}[label=mean]\label{example:mean}(Mean Estimation) Suppose $Z\in \mathbb{R}^p$ is random vector with mean $\theta^*=\ev(Z)$. The loss function becomes  $l(\theta, z)=\frac{1}{2}\|z-\theta\|^2$ with $\theta, z\in \mathbb{R}^p$.
\end{Example}
\begin{Example}[label=linear]\label{example:linear}(Linear Regression) Let the random vector be $Z=(Y, X^\top)^\top$ with $Y\in \mathbb{R}$ and $X\in \mathbb{R}^p$ satisfying $Y=X^\top\theta^*+\epsilon$. Here $\epsilon\in \mathbb{R}$ is the random noise with zero mean. The loss function can be chosen as $l(\theta, z)=\frac{1}{2}(y-x^\top \theta)^2$ with $y\in \mathbb{R}, x, \theta \in \mathbb{R}^p$, and $z=(y, x^\top)^\top$.
\end{Example}
\begin{Example}[label=logistic]\label{example:logistic} (Logistic Regression) Suppose that the observation $Z=(Y, X^\top)^\top$ with $Y\in \{-1, 1\}$ and $X\in \mathbb{R}^p$ satisfying $\pr(Y=y|X=x)=[1+\exp(-yx^\top\theta^*)]^{-1}$. The loss function is $l(\theta, z)=\log(1+\exp(-yx^\top\theta))$ with $y\in \{-1, 1\}, x, \theta \in \mathbb{R}^p$, and $z=(y, x^\top)^\top$.
\end{Example}
\begin{Example}[label=mle]\label{example:likelihood}(Maximal Likelihood Estimation) Let $F_{\theta^*}$ be the distribution of $Z$, and the function form of $F_{\theta^*}$ is known except the value of $\theta^*$. The loss function is the negative log likelihood: $l(\theta, z)=-\log(F_{\theta}(z))$.
\end{Example}

In general, the function form of $L(\theta)$ is unknown, as it relies on the distribution $F_{\theta^*}$. Instead, classical statistical methods  estimate $\theta^*$  based on the sample counterpart of $L(\theta)$ as follows:
\begin{eqnarray}
\tilde{\theta}_T=\argmin_{\theta \in \mathbb{R}^p}\frac{1}{T}\sum_{t=1}^T l(\theta, Z_t), \quad \textrm{s.t.} \quad  B\theta=b,\label{eq:mle}
\end{eqnarray}
where $Z_1,\ldots, Z_T$ are the i.i.d. observations generated from distribution $F_{\theta^*}$. However, the computation of $\tilde{\theta}_T$ in (\ref{eq:mle}) involves calculating a summation among $T$ terms, which is not efficient when sample size $T$ is large. Moreover, in many real-world scenarios,  the observations are collected sequentially in an online fashion. With the growing number of observations,  data storage devices cannot store all the collected observations or there is no enough memory to load the whole dataset. In this case, the classical estimation procedures are not computationally feasible.
\\

Before proceeding, we introduction some  notation. Let $\|v\|=\sqrt{v^\top v}$ denote the Euclidean norm of the vector $v$. For any matrix $A\in \mathbb{R}^{q\times k}$, we define $\|A\|=\sup_{x\in \mathbb{R}^k}\sqrt{x^\top A^\top Ax}$ as its operator norm, $A^-$ as its Moore–Penrose inverse, and $\textrm{rank}(A)$ as its rank.  For two symmetric matrices $V_1, V_2 \in \mathbb{R}^{k\times k}$, we say $V_1\succeq V_2$ if $x^\top V_1x\geq x^\top V_2x$ for all $x\in \mathbb{R}^k$. We use the notation $\cip$ and $\cid$ to denote convergence in probability and in distribution, respectively. For $t\geq 1$, we denote $\mcF_{t}$ as the sigma algebra generated by $\{Z_1,\ldots, Z_t\}$. We denote $\chi^2(k)$  as the chi-square distribution with degree of freedom $k$, and $\chi^2(\delta, k) $  as  the non-central chi-squared distribution with noncentrality parameter $\delta$ and degree of freedom $k$, for positive integer $k$ and positive constant $\delta$.

\section{Projected Polyak–Ruppert averaging}\label{section:PSGD}
To overcome the drawbacks of the classical methods, we consider the following PSGD algorithm. Choosing an initial value $\theta_0\in \mathbb{R}^p$, we recursively update the value as follows:
\begin{equation}
\theta_t=\Pi(\theta_{t-1}-\gamma_t \nabla l(\theta_{t-1}, Z_t)),\label{eq:psgd}
\end{equation}
where $\Pi(\cdot)$ is the projection operator onto the affine set $\{\theta\in \mathbb{R}^p: B\theta=b\}$, and $\gamma_t>0$ is the predetermined learning rate (or step size). The  updating equation in (\ref{eq:psgd}) can be explicitly written in matrix form as
\begin{equation}
\theta_t=c+P[\theta_{t-1}-\gamma_t \nabla l(\theta_{t-1}, Z_t)-c],\nonumber %\label{eq:psgd:matrix}
\end{equation}
where $P\in \mathbb{R}^{p\times p}$ is the orthogonal  projection matrix onto $\textrm{Ker}(B)$, and $c\in \mathbb{R}^p$ is any vector satisfying $Bc=b$. Following \cite{polyak1992acceleration}, we define the APSGD estimate as follows:
\begin{eqnarray}
\overline{\theta}_T=\frac{1}{T}\sum_{t=1}^T\theta_t.\label{eq:apsgd}
\end{eqnarray}
By projection operation in (\ref{eq:psgd}),  the estimate $\overline{\theta}_T$ satisfies (\ref{eq:linear:constraint}). It is worth mentioning that, the average in (\ref{eq:apsgd}) can be updated recursively in an online fashion as 
\begin{eqnarray*}
\overline{\theta}_t=\frac{t-1}{t} \overline{\theta}_{t-1}+\frac{1}{t}\theta_t,
\end{eqnarray*}
which is also obtainable with  a large sample size. To discuss the theoretical properties of $\overline{\theta}_T$, we need the following Assumption.
\begin{Assumption}\label{Assumption:LH}
There exist constants $K, \epsilon>0$ such that the following statements hold.
\begin{enumerate}[label={\normalfont(\roman*)}, wide=0pt]
\item \label{LH:1} The learning rate satisfies $\gamma_t=\gamma t^{-\rho}$, for some constants $\gamma>0$ and $\rho\in (1/2, 1)$.
\item  \label{LH:2} The objective function $L(\theta)$ is convex and continuously differentiable for all $\theta\in \mathbb{R}^p$. Moreover, it is twice continuously differentiable at $\theta=\theta^*$, where $\theta^*$ is the unique minimizer of $L(\theta)$.
\item  \label{LH:3}  For all $\theta, \tilde{\theta}\in \mathbb{R}^p$, the inequality $\|\nabla L(\theta)-\nabla L(\tilde{\theta})\|\leq K\|\theta-\tilde{\theta}\|$ holds.
\item  \label{LH:4} The Hessian matrix $G:=\nabla^2 L(\theta^*)\in \mathbb{R}^{p\times p}$ is positive definite. Furthermore,  the inequality $\|\nabla^2 L(\theta)-\nabla^2 L(\theta^*)\|\leq K\|\theta-\theta^*\|$ holds for all $\theta$ with $\|\theta-\theta^*\|\leq \epsilon$. 
\item  \label{LH:5} For all $\theta \in \mathbb{R}^p$, it holds that $E(\|\nabla l(\theta, Z)\|^2)\leq K(1+\|\theta\|^2)$,  and the matrix $S:=\ev(\nabla l(\theta^*, Z) \nabla l^\top (\theta^*, Z))\in \mathbb{R}^{p\times p}$ is positive definite.
\item \label{LH:6} For all $\theta$ with $\|\theta-\theta^*\|\leq \epsilon$, it holds that $\ev(\|\nabla l(\theta, Z)-\nabla l(\theta^*, Z)\|^2)\leq \delta(\|\theta-\theta^*\|)$, where $\delta(\cdot)$ is a function such that $\delta(v)\to$ as $v\to 0$.
\item \label{LH:6.5} For each $\theta\in \mathbb{R}^p$, there exists a constant $\epsilon_{\theta}>0$ and a measurable function $M_{\theta}(z)$ with $\ev(M_{\theta}(Z))<\infty$ such that 
\begin{eqnarray*}
\sup_{\tilde{\theta}: \|\tilde{\theta}-\theta\|\leq \epsilon_{\theta}}\|\nabla l(\tilde{\theta}, Z)\|\leq M_{\theta}(Z)\quad  \textrm{almost surely.}
\end{eqnarray*}

\item \label{LH:7} The projection matrix $P$ satisfies $P^2=P^\top=P$ and $\textrm{rank}(P)=d$ for some integer $d\in \{0, \ldots, p\}$.
\end{enumerate}
\end{Assumption}
\begin{Remark}
Assumption \ref{Assumption:LH}\ref{LH:1} specifies the learning rate for $t$-th iteration. The learning rate satisfies $\sum_{t=1}^\infty \gamma_t=\infty$ and $\sum_{t=1}^\infty \gamma_t^2<\infty$, which is widely used in literature \cite{polyak1992acceleration, fang2018online, su2018uncertainty}. Assumptions \ref{Assumption:LH}\ref{LH:2}-\ref{Assumption:LH}\ref{LH:6.5} are regularity conditions about the objective function $L(\theta)$ and the lose function $l(\theta, z)$, which are standard and also adopted in \cite{fang2018online}. 
Assumption \ref{Assumption:LH}\ref{LH:7} is to characterize the linear-equality constraint $B\theta^*=b$. In particular, when $P=I$ and $d=p$, the APSGD estimate $\overline{\theta}_T$ in (\ref{eq:apsgd}) becomes the ASGD estimate without projection in \cite{polyak1992acceleration}. 
\end{Remark}

\begin{theorem}\label{theorem:asymptotic:expansion}
Under Assumption \ref{Assumption:LH}, it follows that
\begin{eqnarray*}
\overline{\theta}_T=\theta^*-\frac{1}{T}\sum_{t=1}^T (PGP)^-\zeta_t+o_\p(T^{-1/2}),
\end{eqnarray*}
where  $\zeta_t=\nabla l(\theta_{t-1}, Z_t)-\nabla L(\theta_{t-1})$. Moreover, the following statement holds:
\begin{eqnarray*}
\sqrt{T}(\overline{\theta}_T-\theta^*)\cid N(0, (PGP)^{-}S(PGP)^{-}).
\end{eqnarray*}
\end{theorem}
Theorem \ref{theorem:asymptotic:expansion} provides the asymptotic expansion and limiting distribution of the APSGD estimate $\overline{\theta}_T$. Notice that $\theta_{t-1}\in \mcF_{t-1}$, and $Z_t$ is independent from $\mcF_{t-1}$, so $\ev(\zeta_t|\mcF_{t-1})=0$, which implies that $\zeta_1,\ldots, \zeta_T$ is a martingale-difference process. Under Assumption \ref{Assumption:LH}, we can apply the martingale central limit theorem (e.g., see \cite{Pollard1984convergence}) to derive the limiting distribution. It is worth mentioning the differences and connections between Theorem \ref{theorem:asymptotic:expansion} and the existing results. First, \cite{polyak1992acceleration} considered an unconstrained parameter space and showed that the ASGD estimate is asymptotically distributed as $N(0, G^{-1}SG^{-1})$.  Theorem \ref{theorem:asymptotic:expansion} can be viewed as an extension of \cite{polyak1992acceleration} from $P=I$ to a general projection matrix $P$. Second,  \cite{godichon2017averaged} studied the APSGD estimate when the model parameters are in the interior of the constrained parameter space, and they showed that APSGD have the same limiting distribution as PSGD. However, Theorem \ref{theorem:asymptotic:expansion} reveals the different limiting distributions of APSGD and PSGD in our model. The reason behind this difference is that our model parameter $\theta^*$ is not in the interior of the constrained parameter space $\{\theta\in \mathbb{R}^p : B\theta=b\}$.

Let us revisit examples in previous section and investigate the limiting distributions of the corresponding APSGD estimates.
\begin{Example}[continues=mean]
Suppose the covariance of $Z$ is $\Sigma$. We can verify $\nabla l(\theta, z)=-(z-\theta)$, $\nabla^2 l(\theta, z)=I$, $G=I$, and $S=\Sigma$. So the asymptotic covariance of the APSGD estimate is $P\Sigma P$.
\end{Example}

\begin{Example}[continues=linear]
Suppose $\epsilon$ is independent from $X$ with $\ev(\epsilon)=0$, $\ev(\epsilon^2)=\sigma^2$. It can be verified that $\nabla l(\theta, z)=-(y-x^\top \theta)x$, $\nabla^2 l(\theta, z)=xx^\top$, $G=\ev(XX^\top)$, and $S=\sigma^2 \ev(XX^\top)=\sigma^2 G$. Hence, the APSGD estimate is asymptotically with covariance matrix $\sigma^2 (PGP)^-$.
\end{Example}

\begin{Example}[continues=logistic]
Suppose $\epsilon$ is independent from $X$ with $\ev(\epsilon)=0$, $\ev(\epsilon^2)=\sigma^2$, and $V=\ev(XX^\top)$. It is not difficult to verify that
\begin{eqnarray*}
\nabla l(\theta, z)=\frac{-yx}{1+\exp(yx^\top\theta)},\;\; \nabla^2 l(\theta, z)=\frac{\exp(yx^\top \theta)}{[1+\exp(yx^\top\theta)]^2} xx^\top,\;\; G=S=\ev\bigg(\frac{\exp(X^\top \theta^*)}{[1+\exp(X^\top \theta^*)]^2}XX^\top\bigg).
\end{eqnarray*}
As a consequence, the APSGD estimate is asymptotically normal with covariance matrix $(PGP)^-$.
\end{Example}

\begin{Example}[continues=mle]
Assume almost surely for all $Z$, the map $\theta \to F_\theta(Z)$ is twice continuously differentiable. Due to the properties of log likelihood function, the Fisher information matrix satisfies $I_{\theta^*}:=\ev[\nabla^2 l(\theta^*, Z)]=\ev[\nabla l(\theta^*, Z) \nabla l(\theta^*, Z)]=G=S$. Therefore, we show that the covariance matrix is $(PI_{\theta^*}P)^-$.
\end{Example}

It is worth discussing the role of the constraint (\ref{eq:linear:constraint}) played in the estimation. For this purpose, let us denote $\overline{\theta}_{T,I}$ and $\overline{\theta}_{T,P}$ as the APSGD estimates using projection matrices $I$ and $P$, respectively. By Theorem \ref{theorem:asymptotic:expansion}, their asymptotic covariance matrices are  
\begin{eqnarray*}
V_I:=G^{-1}SG^{-1} \quad \textrm{ and }\quad V_P:=(PGP)^{-}S(PGP)^{-}. 
\end{eqnarray*}
For a general loss function $l(\theta, z)$, the performance  $\overline{\theta}_{T,P}$ is not necessarily better than $\overline{\theta}_{T,I}$. To see this, let us consider a special case of Example \ref{example:mean}. 
\begin{Example}[continues=mean]
Suppose $\theta^*=(\theta_1^*, \theta_2^*)^\top \in \mathbb{R}^2$, $B=(1, -1)$ and $b=(0, 0)^\top$. The linear-equality constraint in (\ref{eq:linear:constraint}) becomes $\theta_1^*=\theta_2^*$. Moreover, we assume $\Sigma=\textrm{Diag}(\sigma^2, 3\sigma^2)$. We can verify that
\begin{eqnarray*}
V_P=\begin{pmatrix}
\sigma^2 & \sigma^2 \\
\sigma^2  &\sigma^2
\end{pmatrix}\quad  \textrm{ and } \quad V_I=\begin{pmatrix}
\sigma^2 &0\\
0&3\sigma^2
\end{pmatrix}.
\end{eqnarray*}
As a consequence, neither $V_P \succeq V_I$ nor $V_I \succeq V_P$ holds.
\end{Example}
However, for a board class of loss functions,  the following Lemma suggests $\overline{\theta}_{T, P}$ is at least as efficient as $\overline{\theta}_{T,I}$. 
\begin{lemma}\label{lemma:more:efficient:theta:P}
Under Assumption \ref{Assumption:LH}, if $S=c G$ for some constant $c>0$, then $V_I=c G^{-1}$ and $V_P=c (PGP)^-$. Moreover, it follows that $V_I\succeq V_P$, and the equality $V_I=V_P$ holds if and only if $P=I$.
\end{lemma}
Lemma \ref{lemma:more:efficient:theta:P} indicates that, under an additional condition, the estimation performance of $\overline{\theta}_{T,P}$ is improved by utilizing the additional information in (\ref{eq:linear:constraint}). The additional condition $S=cG$ holds for many popular models, including Examples \ref{example:linear}-\ref{example:likelihood}. In particular, for the negative log likelihood  loss function in Example \ref{example:likelihood}, the asymptotic covariance matrix $(PI_{\theta^*}P)^-$ coincides the Cram\'er–Rao lower bound for constrained maximal likelihood model (e.g., see \cite{gorman1990lower, moore2008maximum}).

To apply Theorem \ref{theorem:asymptotic:expansion}, the unknown covariance matrix needs to be estimated.  For this purpose, the following regularity conditions on $ l(\theta, z)$  are imposed.
\begin{Assumption}\label{Assumption:var}
There exists a constant $\epsilon>0$ such that, for each $\theta$ with $\|\theta-\theta^*\|\leq \epsilon$, the function $\theta \to l(\theta, Z)$ has a continuous Hessian matrix $\nabla^2 l(\theta, Z)$ almost surely. Moreover, there exists a measurable function $M(z)$ with $\ev(M(Z))<\infty$ satisfying $\|\nabla^2 l(\theta, Z)\|\leq M(Z)$ for all  $\theta$ with $\|\theta-\theta^*\|\leq \epsilon$ almost surely.
\end{Assumption}
The existence of the second-order derivatives of $\theta \to l(\theta, z)$ in Assumption \ref{Assumption:var} is to estimate $G=\nabla^2 L(\theta^*)$ based on its sample counterpart, while the dominating function $M(Z)$ is required to allow changing the order of the gradient operator and expectation, namely,  $\nabla^2 \ev[l(\theta^*, Z)]= \ev[\nabla^2l(\theta^*, Z)]$.  To estimate the covariance matrix, let us define 
\begin{eqnarray}
\hat{G}_{T}=\frac{1}{T}\sum_{t=1}^T \nabla^2 l(\overline{\theta}_t, Z_t), \quad \hat{S}_{T}=\frac{1}{T}\sum_{t=1}^T \nabla l(\overline{\theta}_t, Z_t)\nabla l^\top (\overline{\theta}_t, Z_t),\label{eq:G:hat:S:hat}
\end{eqnarray}
which both can be recursively  calculate by
\begin{eqnarray*}
\hat{G}_{t}=\frac{t-1}{t}\hat{G}_{t-1}+\frac{1}{t}\nabla^2 l(\overline{\theta}_t, Z_t),\quad \hat{S}_{t}=\frac{t-1}{t}\hat{S}_{t-1}+\frac{1}{t} \nabla l(\overline{\theta}_t, Z_t)\nabla l^\top (\overline{\theta}_t, Z_t).
\end{eqnarray*}
The following lemma provides a consistent estimate for the covariance matrix.
\begin{lemma}\label{lemma:estimation:variance}
Under Assumptions \ref{Assumption:LH} and \ref{Assumption:var}, it follows that 
$(P\hat{G}_{T}P)^-\hat{S}_{T}(P\hat{G}_{T}P)^-=(P{G}P)^-S(P{G}P)^-+o_\p(1).$
\end{lemma}
Combining Theorem \ref{theorem:asymptotic:expansion} with Lemma \ref{lemma:estimation:variance}, we can construct an $(1-\alpha)\times 100\%$  confidence interval for the function $g(\theta^*)$ as
\begin{eqnarray}
g(\overline{\theta}_T)\pm z_{\alpha/2} \sqrt{\frac{\nabla g^\top(\overline{\theta}_T)(P\hat{G}_{T}P)^-\hat{S}_{T}(P\hat{G}_{T}P)^-\nabla g(\overline{\theta}_T)}{T}},\label{eq:ci:g:theta}
\end{eqnarray} 
where $z_{\alpha/2}$ is the $\alpha/2 \times 100\%$ upper quartile of standard normal. Since  $\overline{\theta}_T$, $\hat{G}_{T}$ and $\hat{S}_{T}$ can be computed in an online fashion, so is the confidence interval in (\ref{eq:ci:g:theta}).

\section{Specification test}\label{section:test}
As a byproduct of Theorem \ref{theorem:asymptotic:expansion}, we propose a specification test  for the constraint in (\ref{eq:linear:constraint}).  Specifically, we aim to test the following hypotheses:
\begin{eqnarray*}
H_0: B\theta^*=b \quad \textrm{ vs.} \quad H_1: B\theta^*=b+\beta \textrm{ for some } \beta\neq 0.
\end{eqnarray*}
For this purpose, we define the test statistic
\begin{equation}\label{eq:test:statistic}
\kappa_T=T(\overline{\theta}_{T, P}-\overline{\theta}_{T, I})^\top \hat{W}^-(\overline{\theta}_{T, P}-\overline{\theta}_{T, I}).
\end{equation}
Here $\hat{W}=(I-P)\hat{G}_{T, I}^{-1}\hat{S}_{T,I} \hat{G}_{T,I}^{-1}(I-P)$ is a weight matrix with $\hat{G}_{T,I}$ and $\hat{S}_{T,I}$ being the matrices in (\ref{eq:G:hat:S:hat}) calculated using projection matrix $I$. Essentially, $\hat{W}$ estimates the weight matrix $W=(I-P)G^{-1}SG^{-1}(I-P)$.  The idea of the proposed test statistic in (\ref{eq:test:statistic}) is simple and straightforward. Under $H_0$, both  $\overline{\theta}_{T, P}$ and $\overline{\theta}_{T, I}$ consistently estimate $\theta^*$. Hence, their difference, as well as $\kappa_T$, should be around zero. However, under $H_1$, due to model misspecification,  $\overline{\theta}_{T, P}$  is inconsistent, and the difference $\overline{\theta}_{T, P}-\overline{\theta}_{T, I}$ does not vanish. Based on (\ref{eq:test:statistic}), we propose the following asymptotic size $\alpha$ testing procedure: 
\begin{eqnarray}
\textrm{reject $H_0$\quad if  \;  $\kappa_T>\chi^2_{\alpha}(p-d)$, }\label{eq:test:procedure}
\end{eqnarray}
where $\chi^2_{\alpha}(p-d)$ is the $\alpha \times 100\%$ upper quartile of $\chi^2$ distribution with degree $p-d$. The following theorem reveals the limiting behavior of the statistic $\kappa_T$ and the validity of the proposed testing procedure.
\begin{theorem}\label{theorem:specification:test}
Suppose Assumptions \ref{Assumption:LH} and \ref{Assumption:var} are satisfied. Then the following  statements are true:
\begin{enumerate}[label={\normalfont(\roman*)}, wide=0pt]
\item Under $H_0: B\theta^*=b$, the convergence $\kappa_T\cid \chi^2(p-d)$ holds.
\item Under $H_1: B\theta^*=b+\beta$ for some  $\beta\neq 0$, it follows that $\kappa_T\to \infty$ in probability.
\item Under $H_a: B\theta^*=b+\frac{\beta}{\sqrt{T}}$ for some $\beta\neq 0$, it holds that $\kappa_T\cid \chi^2 (\mu^\top W^- \mu, p-d)$, where $\mu\in \mathbb{R}^p$ is any vector satisfying $B\mu=\beta$.
\end{enumerate}  
As a consequence, for any $\alpha\in (0, 1)$, it follows that 
\begin{eqnarray*}
\lim_{T\to \infty}{\normalfont \pr}(\kappa_T>\chi^2_{\alpha}(p-d)| H_0 )=\alpha,  \quad \lim_{T\to \infty}{\normalfont \pr}(\kappa_T>\chi^2_{\alpha}(p-d)| H_1 )=1.
\end{eqnarray*}
\end{theorem}
Theorem \ref{theorem:specification:test} provides the asymptotic distributions of $\kappa_T$   under null hypothesis $H_0$ and local alternative hypothesis $H_a$, which are chi-square and noncentral chi-squared, respectively. Moreover, it shows that $\kappa_T$ will diverge under alternative hypothesis $H_1$.  Consequently, it verifies that testing procedure in (\ref{eq:test:procedure})  is consistent and has an asymptotic size $\alpha$.

\section*{Acknowledgments}
The authors gratefully acknowledge the constructive comments and suggestions from the Editor-in-Chief  Dr. Dietrich von Rosen, an associate editor, and two anonymous referees. 
Zuofeng Shang acknowledges supports by NSF DMS-1764280 and DMS-1821157.
%\section*{Acknowledgments}

%We thank the Editor, Associate Editor and referees, as well as our financial sponsors.

\section*{Appendix}
In the appendix, we collect all the mathematical proofs of the main theorems and  related lemmas. Section \ref{section:basic:lemma} provides some preliminary lemmas, while the Sections \ref{section:proof:theorem:1}-\ref{section:proof:theorem:2} proves Theorem \ref{theorem:asymptotic:expansion}, Lemma \ref{lemma:more:efficient:theta:P}, Lemma  \ref{lemma:estimation:variance}, and Theorem \ref{theorem:specification:test}.
%\newpage
%\appendix

\setcounter{subsection}{0}
\renewcommand{\thesubsection}{A.\arabic{subsection}}
\setcounter{section}{0}
\renewcommand{\thesection}{A.\arabic{section}}
\setcounter{subsubsection}{0}
\renewcommand{\thesubsubsection}{\textbf{A.\arabic{subsection}.\arabic{subsubsection}}}
\setcounter{lemma}{0}
\renewcommand{\thelemma}{A.\arabic{lemma}}
\setcounter{proposition}{0}
\renewcommand{\theproposition}{A.\arabic{proposition}}
\setcounter{Assumption}{0}
\renewcommand{\theAssumption}{B\arabic{Assumption}}

\section{Preliminary lemmas}\label{section:basic:lemma}

\begin{lemma}\label{lemma:lemma:1:polyak}
Let $H\in \mathbb{R}^{p\times p}$ be a  positive definite matrix and  suppose that $\gamma_t=\gamma t^{-\rho}$ for some constants $\gamma>0, \rho\in (1/2, 1)$.  Let us define squared matrices 
\begin{align*}
W_j^{j}&=I,\quad W_j^{t}=(I-\gamma_{t-1} H)W_j^{t-1}=\cdots=\prod_{k=j}^{t-1}(I-\gamma_k H) \quad \textrm{ for } t\geq j,\\
\overline{W}_j^t&=\gamma_j \sum_{i=j}^{t-1}W_j^i=\gamma_j \sum_{i=j}^{t-1}\prod_{k=j}^{i-1}(I-\gamma_k H).
\end{align*}
Then the following statements hold:
\begin{enumerate}[label={\normalfont(\roman*)}, wide=0pt]
\item There are constants $K>0$ such that $\|\overline{W}_j^t\|\leq  K \textrm{ for all } j \textrm{ and all } t\geq j$.
\item $\frac{1}{t}\sum_{j=0}^{t-1}\|\overline{W}_j^t-H^{-1}\|\to 0$ as $t\to \infty$.
\end{enumerate}
\end{lemma}
\begin{proof}[\textbf{\upshape Proof:}] 
This is Lemma  1 of \cite{polyak1992acceleration}.
\end{proof}

\begin{lemma}\label{lemma:matrices:same:eigenvector}
Let $A\in \mathbb{R}^{p\times p}$ be a  positive definite matrix and $P$ be a projection matrix such that $P=P^\top$ and $\textrm{rank}(P)=d$.  Then there exists an orthonormal  matrix $U\in \mathbb{R}^{p\times p}$ such that
\begin{eqnarray*}
U^\top P U=\begin{pmatrix}
I_d&0\\
0&0
\end{pmatrix}, \quad U^\top PAP U=\begin{pmatrix}
\Omega_d&0\\
0&0
\end{pmatrix}, \quad U^\top (PAP)^- U=\begin{pmatrix}
\Omega_d^{-1}&0\\
0&0
\end{pmatrix},
\end{eqnarray*}
where $I_d \in \mathbb{R}^{d\times d}$ is the identity matrix, and $\Omega_d$ is a diagonal matrix with diagonal elements $\rho_1,\ldots, \rho_d>0$. Moreover, it follows that  $(PAP)^-P=(PAP)^-$, $P(PAP)^-=(PAP)^-$ and $(PAP)^-(PAP)x=x$ for all $x$ satisfying $Px=x$.
\end{lemma}
\begin{proof}[\textbf{\upshape Proof:}] 
For any $x\in \mathbb{R}^p$ with $PAPx=0$, it holds that $x^\top PAPx=0$ and $Px=0$ by the positive definiteness of $A$. Clearly, $Px=0$ implies $PAPx=0$. Therefore, we conclude that $
\textrm{Ker}(PAP)=\textrm{Ker}(P)$ and $\textrm{rank}(PAP)=\textrm{rank}(P)=d$.

For simplicity, we denote $S=PAP$. By direct examination, $S$ and $P$ are diagonalisable, and they commute. By simple linear algebra,  there exist eigenvectors  $u_1, u_2, \ldots, u_p$ that simultaneously diagonalize $P$ and $S$. W.L.O.G, we assume  $Pu_i=u_i$ for $i\in \{1,\ldots, d\}$ and $Pu_i=0$ for $i\in \{d+1,\ldots, p\}$. We further assume $\rho_1,\rho
_2,\ldots, \rho_{p}$ to be the eigenvalues of $S$ corresponding to the eigenvectors $u_1,u_2,\ldots, u_p$. By the above notation, it shows that
\begin{eqnarray*}
\rho_iu_i=Su_i=PAPu_i=0 \textrm{ for } i\in \{d+1,\ldots, p\}.
\end{eqnarray*}
Since $\textrm{rank}(S)=d$, we conclude that $\rho_i>0$ for $i\in \{1,\ldots, d\}$.   As a consequence, $U=(u_1,\ldots, u_p)$ and $\Omega_d=\textrm{Diag}(\rho_1,\ldots, \rho_p)$ will be the desired choices. Moreover, it is not difficult to verify that
\begin{eqnarray*}
(PAP)^-P=U\begin{pmatrix}
\Omega_d^{-1}&0\\
0&0
\end{pmatrix}U^\top U\begin{pmatrix}
I_d&0\\
0&0
\end{pmatrix}U^\top=U\begin{pmatrix}
\Omega_d^{-1}&0\\
0&0
\end{pmatrix}U^\top=(PAP)^-.
\end{eqnarray*}
Similarly, we can prove that $P(PAP)^-=(PAP)^-$. Suppose that $x$ satisfies $Px=x$, then $x=\sum_{i=1}^d c_i u_i$ for some $c_1,\ldots, c_d\in \mathbb{R}$. As a consequence, it follows that $PAPx=PAP \sum_{i=1}^d c_i u_i=\sum_{i=1}^d c_i \rho_i u_i.$
Notice that $PAP=\sum_{i=1}^d \rho_i u_i u_i^\top$ and $(PAP)^{-}=\sum_{i=1}^d \rho_i^{-1} u_i u_i^\top$, we have $(PAP)^-(PAP)x=\sum_{i=1}^d \rho_i^{-1} u_i u_i^\top\sum_{i=1}^d c_i \rho_i u_i=\sum_{i=1}^d c_i  u_i=x.$
%Moreover, we have
%\begin{eqnarray*}
%\rho_iu_i=Su_i=PAPu_i=\begin{cases}
%PAu_i & \textrm{ if } i=1,\ldots, d\\
%0&\textrm{ if } i=d+1,\ldots, p,
%\end{cases}
%\end{eqnarray*}
%which implies $PAP=PA$.
\end{proof}

\begin{lemma}\label{lemma:lemma:matrices:iteration:my}
Under Assumption \ref{Assumption:LH}, it follows that
\begin{eqnarray*}
\lim_{t\to \infty}\frac{1}{t}\sum_{j=0}^{t-1}\|\gamma_j\sum_{k=j}^{t-1}\prod_{i=j+1}^kP(I-\gamma_i G)P-(PGP)^{-}\|=0.
\end{eqnarray*}
Moreover, there is a constant $K>0$ such that  $\|\gamma_j \sum_{k=j}^{t-1}\prod_{i=j+1}^{k}P(I-\gamma_i G)P\|\leq K$ for all $j$ and all $t\geq j$.
\end{lemma}
\begin{proof}[\textbf{\upshape Proof:}] 
 Since $G$ is positive definite by Assumption \ref{Assumption:LH}\ref{LH:4}, it follows from Lemma  \ref{lemma:matrices:same:eigenvector} that
\begin{eqnarray*}
U^\top P(I-\gamma_i G)PU=U^\top P U- \gamma_i  U^\top PGPU=\begin{pmatrix}
I_d-\gamma_i \Omega_d&0\\
0&0
\end{pmatrix},
\end{eqnarray*}
where $U$ is an orthonormal matrix, $I_d \in \mathbb{R}^{d\times d}$ is the identity matrix, and $\Omega_d \in \mathbb{R}^{d\times d}$ is a diagonal  and positive definite matrix. As a consequence, we have
\begin{eqnarray*}
\prod_{i=j+1}^k  P(I-\gamma_i G)P=U\begin{pmatrix}
\prod_{i=j+1}^k(I_d-\gamma_i \Omega_d)&0\\
0&0
\end{pmatrix}U^\top.
\end{eqnarray*}
By Lemma \ref{lemma:lemma:1:polyak}, we have 
\begin{eqnarray*}
\lim_{t\to \infty}\frac{1}{t}\sum_{j=0}^{t-1}\|\gamma_j\sum_{k=j}^{t-1}\prod_{i=j+1}^k(I_d-\gamma_i \Omega_d)-\Omega_d^{-1}\|=0,
\end{eqnarray*}
which further leads to the first  statement according to Lemma \ref{lemma:matrices:same:eigenvector}. Applying Lemma \ref{lemma:matrices:same:eigenvector} again, we obtain the second conclusion.
\end{proof}

\begin{lemma}\label{lemma:su:zhu:10}
Let $c_1$ and $c_2$  be arbitrary positive constants. Support that $\gamma_t=\gamma t^{-\rho}$ for some constants $\gamma>0$ and $\rho\in (1/2, 1)$. Moreover, assume a sequence $\{B_t\}_{t=1}^\infty$ satisfies 
\begin{eqnarray*}
B_t\leq \frac{\gamma_{t-1}(1-c_1\gamma_t)}{\gamma_t}B_{t-1}+c_2\gamma_t.
\end{eqnarray*}
Then $\sup_{1\leq t <\infty}B_t<\infty$.
\end{lemma}
\begin{proof}[\textbf{\upshape Proof:}] 
This Lemma A.10 in \cite{su2018uncertainty}.
\end{proof}

\begin{lemma}\label{su:lemma:B1}
Let $F(x)$ be a differentiable convex function defined on $\mathbb{R}^p$ with an unique minimizer $x^*$. Suppose there exist constants $\rho, r>0$ such that  $x\to F(x)-\frac{\rho}{2}\|x\|^2$ is convex for all $x$ with $\|x-x^*\|\leq r$. Then for  all $x\in \mathbb{R}^p$, it holds that $(x-x^*)^\top\nabla F(x)\geq \rho\|x-x^*\|\min\{\|x-x^*\|, r\}.$
\end{lemma}
\begin{proof}[\textbf{\upshape Proof:}] 
This is Lemma B.1 in \cite{su2018uncertainty}.
\end{proof}

\section{Proof of Theorem \ref{theorem:asymptotic:expansion}}\label{section:proof:theorem:1}
Before stating  technical lemmas, we sketch the proof of  Theorem \ref{theorem:asymptotic:expansion}. By iteration formula in (\ref{eq:psgd}), we have
\begin{eqnarray}
\theta_t=c+P(\theta_{t-1}-\gamma_ty_t-c), \quad y_t=\nabla l(\theta_{t-1}, Z_t)=\nabla L(\theta_{t-1})+[\nabla l(\theta_{t-1}, Z_t)-\nabla L(\theta_{t-1})]:=R(\theta_{t-1})+\zeta_t,\label{eq:iteration:proof}
\end{eqnarray}
where $c\in \mathbb{R}^p$ is any vector satisfying $Bc=b$. Let $\Delta_t=\theta_t-\theta^*$. Since $P(\theta^*-c)=\theta^*-c$, it follows that
\begin{align*}
\Delta_t=\theta_t-\theta^*&=c+P(\theta_{t-1}-\gamma_ty_t-c)-\theta^*=c+P(\Delta_{t-1}-\gamma_ty_t+\theta^*-c)-\theta^*=P\Delta_{t-1}-\gamma_t Py_t\\
&=P\Delta_{t-1}-\gamma_t P R(\theta_{t-1})-\gamma_t P\zeta_t=P\Delta_{t-1}-\gamma_t P G\Delta_{t-1}-\gamma_t P\zeta_t-\gamma_tP(R(\theta_{t-1})-G\Delta_{t-1})\\
&=P(I-\gamma_t  G)\Delta_{t-1}-\gamma_t P\zeta_t-\gamma_tP(R(\theta_{t-1})-G\Delta_{t-1})\\
&=\bigg[\prod_{j=1}^t P(I-\gamma_j G)\bigg]\Delta_0+\sum_{j=1}^t \bigg[\prod_{i=j+1}^t P(I-\gamma_i G)P\bigg]\gamma_j P\zeta_j+\sum_{j=1}^t \bigg[\prod_{i=j+1}^t P(I-\gamma_i G)P\bigg]\gamma_j P(R(x_{j-1})-G\Delta_{j-1}).
\end{align*}
Taking average, we show that
\begin{align}
\frac{1}{T}\sum_{t=1}^T \Delta_t=&\frac{1}{T}\sum_{t=1}^T\bigg[\prod_{j=1}^t P(I-\gamma_j G)\bigg]\Delta_0+ \frac{1}{T}\sum_{t=1}^T\sum_{j=1}^t \bigg[\prod_{i=j+1}^t P(I-\gamma_i G)P\bigg]\gamma_j P\zeta_j\nonumber\\
&+\frac{1}{T}\sum_{t=1}^T\sum_{j=1}^t \bigg[\prod_{i=j+1}^t P(I-\gamma_i G)P\bigg]\gamma_j P(R(x_{j-1})-G\Delta_{j-1}):=S_1+S_2+S_3.\label{eq:definition:s1:s2:s3}
\end{align}
In Lemmas \ref{lemma:s1s2:asymptotic:expansion} and \ref{lemma:bound:3}, we will show that 
\begin{eqnarray}
S_1+S_2=\frac{1}{T}(PGP)^-\sum_{t=1}^T \zeta_t+o_\p(T^{-1/2}),\quad S_3=o_\p(T^{-1/2}).\nonumber
\end{eqnarray}
Finally, we prove the asymptotic normality based on martingale C.L.T.  in Lemma \ref{lemma:verify:clt}.

\begin{lemma}\label{lemma:PR:equal:R}
Under Assumption \ref{Assumption:LH}, the following statements hold for some constants $\epsilon, K>0$.
\begin{enumerate}[label={\normalfont(\roman*)}, wide=0pt]
\item  \label{lemma:PR:equal:R:1} $(\theta-\theta^*)^\top R(\theta)\geq  \epsilon\|\theta-\theta^*\|\min\{\|\theta-\theta^*\|, \epsilon\}$ for all $\theta \in \mathbb{R}^p$.
\item \label{lemma:PR:equal:R:2} $\ev(\zeta_t| \mcF_{t-1})=0$.
\item \label{lemma:PR:equal:R:3} $\ev(\|\zeta_t\|^2| \mcF_{t-1})\leq  K(1+\|\theta_{t-1}\|^2)$ almost surely.
\item \label{lemma:PR:equal:R:4} $\|R(\theta)\|^2\leq  K(1+\|\theta\|^2)$.
\item \label{lemma:PR:equal:R:5} $\|R(\theta)-G(\theta-\theta^*)\|\leq K\|\theta-\theta^*\|^2$ for all $\theta$ with $\|\theta-\theta^*\|\leq \epsilon$.
\end{enumerate}
\end{lemma}
\begin{proof}[\textbf{\upshape Proof:}] 

For statement \ref{lemma:PR:equal:R:1}, by Assumption \ref{Assumption:LH}\ref{LH:4}, we know  $L(\theta)$ satisfies the conditions in Lemma \ref{su:lemma:B1} with some $\rho, r>0$. Therefore, it follows that
\begin{eqnarray*}
(\theta-\theta^*)^\top R(\theta)=(\theta-\theta^*)^\top \nabla L(\theta)\geq \epsilon\|\theta-\theta^*\|\min\{\|\theta-\theta^*\|, \epsilon\},
\end{eqnarray*}
where $\epsilon=\min\{\rho, r\}$.

For statement \ref{lemma:PR:equal:R:2}, since $\theta_{t-1}\in \mcF_{t-1}$ and $Z_t$ is independent from $\mcF_{t-1}$, we have $\ev(\nabla l(\theta_{t-1}, Z_t)|\mcF_{t-1})=\nabla L(\theta_{t-1})$.

Similarly, by  Assumption \ref{Assumption:LH}\ref{LH:5},  the statement \ref{lemma:PR:equal:R:3} follows from the  inequality below: $$\ev(\|\zeta_t\|^2| \mcF_{t-1})=\ev[\|\nabla l(\theta_{t-1}, Z_t)-\nabla L(\theta_{t-1})\|^2| \mcF_{t-1}]\leq \ev(\|\nabla l(\theta_{t-1}, Z)\|^2| \mcF_{t-1})\leq K(1+\|\theta_{t-1}\|^2).$$

Statement \ref{lemma:PR:equal:R:4} follows, as  $\|R(\theta)\|^2=\|\nabla L(\theta)\|^2=\|\nabla L(\theta)-\nabla L(\theta^*)\|^2\leq K^2\|\theta-\theta^*\|^2\leq 2K^2(\|\theta^*\|^2+\|\theta\|^2)$ by Assumption \ref{Assumption:LH}\ref{LH:3}.

To prove statement  \ref{lemma:PR:equal:R:5}, by Assumption \ref{Assumption:LH}\ref{LH:4} and  Taylor expansion, we have
\begin{align*}
\|R(\theta)-G(\theta-\theta^*)\|&=\|R(\theta)-R(\theta^*)-G(\theta-\theta^*)\|=\|R(\theta)-R(\theta^*)-\nabla^2 L(\theta^*)(\theta-\theta^*)\|\\
&=\|\nabla^2 L(\tilde{\theta})(\theta-\theta^*)-\nabla^2 L(\theta^*)(\theta-\theta^*)\|\leq K\|\theta-\theta^*\|^2,\; \textrm{ for all } \theta \textrm{ with } \|\theta-\theta^*\|\leq \epsilon,
\end{align*}
where $\tilde{\theta}$ is a vector between $\theta$ and $\theta^*$.
\end{proof}

\begin{lemma}\label{lemma:V:compact:range}
Suppose  Assumption \ref{Assumption:LH} holds. Then there exists a constant $K>0$ such that
\begin{eqnarray*}
\|R(\theta)-G(\theta-\theta^*)\|\leq K\|\theta-\theta^*\|^2\; \textrm{ for all } \theta\in \mathbb{R}^p.
\end{eqnarray*}
\end{lemma}
\begin{proof}[\textbf{\upshape Proof:}] 
By Assumptions  \ref{Assumption:LH}\ref{LH:3} and \ref{Assumption:LH}\ref{LH:4}, we have
\begin{align*}
\|R(\theta)-G(\theta-\theta^*)\|&=\|R(\theta)-R(\theta^*)-G(\theta-\theta^*)\|\leq \|R(\theta)-R(\theta^*)\|+\|G(\theta-\theta^*)\|\leq  (K+\|G\|)\|\theta-\theta^*\|\\
&\leq (K+\|G\|)\|\theta-\theta^*\|^2/\epsilon,\; \textrm{ for all } \theta \textrm{ with } \|\theta-\theta^*\|\geq \epsilon.
\end{align*}
Combining with  statement \ref{lemma:PR:equal:R:5} in Lemma \ref{lemma:PR:equal:R}, we complete the proof.
\end{proof}

\begin{lemma}\label{lemma:xt:to:xstar}
Under Assumption \ref{Assumption:LH}, it holds that $\lim_{t\to \infty}\theta_t=\theta^*$ almost surely.
\end{lemma}
\begin{proof}[\textbf{\upshape Proof:}] 
Notice that $\theta_t-c\in \textrm{Ker}(B)$ for all $t\geq 1$, so it follows that
\begin{eqnarray*}
\theta_t-\theta^*=c+P[\theta_{t-1}-\gamma_t\nabla l(\theta_{t-1}, Z_t)-c]-\theta^*=\theta_{t-1}-\theta^*-\gamma_t P[R(\theta_{t-1})+\zeta_t].
\end{eqnarray*}
Moreover $P\Delta_{t}=\Delta_{t}$ for $t\geq 1$, we have
\begin{align}
\|\Delta_t\|^2&=\|\Delta_{t-1}-\gamma_tPR(\theta_{t-1})-\gamma_t P\zeta_t\|^2\nonumber=\|\Delta_{t-1}-\gamma_tPR(\theta_{t-1})\|^2-2\gamma_t\Delta_{t-1}^\top P\zeta_t+2\gamma_t^2R^\top(\theta_{t-1})P\zeta_t+ \gamma_t^2\|P\zeta_t\|^2\nonumber\\
&=\|\Delta_{t-1}\|^2+\gamma_t^2\|PR(\theta_{t-1})\|^2-2\gamma_t\Delta_{t-1}^\top R(\theta_{t-1})\nonumber-2\gamma_t\Delta_{t-1}^\top\zeta_t+2\gamma_t^2R^\top(\theta_{t-1})P\zeta_t+ \gamma_t^2\|P\zeta_t\|^2\nonumber\\
&\leq \|\Delta_{t-1}\|^2+\gamma_t^2\|R(\theta_{t-1})\|^2-2\gamma_t\Delta_{t-1}^\top R(\theta_{t-1})-2\gamma_t\Delta_{t-1}^\top\zeta_t+2\gamma_t^2R^\top(\theta_{t-1})P\zeta_t+ \gamma_t^2\|\zeta_t\|^2, \label{eq:lemma:xt:to:xstar:eq1}
\end{align}
for all $t\geq 2$.
Taking conditional expectation on both sides of (\ref{eq:lemma:xt:to:xstar:eq1}) and by Lemma \ref{lemma:PR:equal:R}, we show that there exist constants $K,\epsilon>0$ such that
\begin{align}
\ev(\|\Delta_t\|^2|\mcF_{t-1})&\leq\|\Delta_{t-1}\|^2-2\gamma_t\Delta_{t-1}^\top  R(\theta_{t-1})+\gamma_t^2\|R(\theta_{t-1})\|^2+\gamma_t^2\ev(\|\zeta_t\|^2|\mcF_{t-1})\nonumber\\
&\leq\|\Delta_{t-1}\|^2-2\gamma_t\Delta_{t-1}^\top  R(\theta_{t-1})+\gamma_t^2K(1+\|\theta_{t-1}\|^2)+\gamma_t^2K(1+\|\theta_{t-1}\|^2)\nonumber\\
&=\|\Delta_{t-1}\|^2-2\gamma_t\Delta_{t-1}^\top  R(\theta_{t-1})+2\gamma_t^2K(1+\|\theta_{t-1}\|^2)\nonumber\\
&\leq \|\Delta_{t-1}\|^2+2\gamma_t^2K(1+2\|\theta^*\|^2+2\|\Delta_{t-1}\|^2)-2\gamma_t\Delta_{t-1}^\top  R(\theta_{t-1})\nonumber\\
&=(1+4\gamma_t^2K)\|\Delta_{t-1}\|^2+2\gamma_t^2K(1+2\|\theta^*\|^2)-2\gamma_t\Delta_{t-1}^\top  R(\theta_{t-1})\nonumber\\
&\leq(1+4\gamma_t^2K)\|\Delta_{t-1}\|^2+2\gamma_t^2K(1+2\|\theta^*\|^2)-2\gamma_t \epsilon\|\Delta_{t-1}\|\min\{\|\Delta_{t-1}\|, \epsilon\}.\label{eq:lemma:xt:to:xstar:eq2}
\end{align}
Since $\sum_{t=1}^\infty \gamma_t=\infty$ and  $\sum_{t=1}^\infty \gamma_t^2<\infty$,  applying Robbins-Siegmund Theorem (e.g., see \cite{robbins1971convergence}), we have $\|\Delta_t\|^2\to V$ almost surely for some random variable $V$, and
\begin{eqnarray*}
\sum_{t=1}^\infty 2\gamma_t \epsilon\|\Delta_{t-1}\|\min\{\|\Delta_{t-1}\|, \epsilon\}<\infty \textrm{ almost surely.}
\end{eqnarray*}
As a consequence, it follows that $\lim_{t\to \infty}\|\Delta_{t-1}\|\to 0 \textrm{ almost surely.}$
\end{proof}

\begin{lemma}\label{lemma:convergence:rate:xt}
Suppose Assumption \ref{Assumption:LH} holds. Then for any $M>0$, there exists a constant $K_M>0$ such that
\begin{eqnarray*}
 \ev[\|\theta_t-\theta^*\|^2I(\tau_M>t)]\leq K_M \gamma_t \; \textrm{ for all  } t\geq 0,
\end{eqnarray*}
where $\tau_M=\inf\{i\geq 1: \|\theta_i-\theta^*\|>M\}$ is a stopping time.
\end{lemma}
\begin{proof}[\textbf{\upshape Proof:}] 
By Lemma \ref{lemma:xt:to:xstar}, for any $\delta>0$, there exists a $M>0$ such that $\pr(\sup_{1\leq t <\infty}\|\theta_t-\theta^*\|\leq M)\geq 1-\delta.$ Notice $\{\tau_M>t\}\in \mcF_t$ and on event $\{\tau_M>t\}$, $\|\theta_1-\theta^*\|, \ldots, \|\theta_t-\theta^*\|$ are bounded by $M$, using (\ref{eq:lemma:xt:to:xstar:eq1}),  Lemmas \ref{lemma:PR:equal:R} and  \ref{lemma:V:compact:range}, we have
\begin{align*}
\|\Delta_t\|^2I(\tau_M>t)&\leq\|\Delta_t\|^2I(\tau_M>t-1) \\
&\leq \bigg(\|\Delta_{t-1}\|^2+\gamma_t^2\|R(\theta_{t-1})\|^2-2\gamma_t\Delta_{t-1}^\top R(\theta_{t-1})-2\gamma_t\Delta_{t-1}^\top\zeta_t+2\gamma_t^2R^\top(\theta_{t-1})P\zeta_t+ \gamma_t^2\|\zeta_t\|^2\bigg)I(\tau_M>t-1). 
\end{align*}
By similar calculation in (\ref{eq:lemma:xt:to:xstar:eq2}), we show that there exist constants $K,\epsilon>0$ such that
\begin{align*}
\ev(\|\Delta_t\|^2I(\tau_M>t)|\mcF_{t-1})&\leq \bigg((1+4\gamma_t^2K)\|\Delta_{t-1}\|^2I(\tau_M>t-1)+2\gamma_t^2K(1+2\|\theta^*\|^2)-2\gamma_t \epsilon\|\Delta_{t-1}\|\min\{\|\Delta_{t-1}\|, \epsilon\}\bigg)I(\tau_M>t-1).
%&\leq (1+4\gamma_t^2K)\|\Delta_{t-1}\|^2I(\tau_M>t-1)+2\gamma_t^2K(1+2\|\theta^*\|^2)I(\tau_M>t-1)-2\gamma_t \epsilon\|\Delta_{t-1}\|\min\{\|\Delta_{t-1}\|, \epsilon\}I(\tau_M>t-1).
\end{align*}
Notice that  $\|\Delta_{t-1}\|\min\{\|\Delta_{t-1}\|, \epsilon\}= \|\Delta_{t-1}\|^2$ if $\Delta_{t-1}\|\leq \epsilon$, and $\|\Delta_{t-1}\|\min\{\|\Delta_{t-1}\|, \epsilon\}= \|\Delta_{t-1}\|\epsilon\geq \|\Delta_{t-1}\|^2\epsilon/M$ if $\epsilon<\|\Delta_{t-1}\|\leq M$,
we conclude that
\begin{align*}
\ev(\|\Delta_t\|^2I(\tau_M>t)|\mcF_{t-1})&\leq \bigg((1+4\gamma_t^2K)\|\Delta_{t-1}\|^2+2\gamma_t^2K(1+2\|\theta^*\|^2)-2\gamma_t \epsilon^2 M^{-1}\|\Delta_{t-1}\|^{2}\bigg)I(\tau_M>t-1)\\
&\leq (1-2\gamma_t \epsilon^2 M^{-1}+4\gamma_t^2K)\|\Delta_{t-1}\|^2I(\tau_M>t-1)+2\gamma_t^2K(1+2\|\theta^*\|^2).
\end{align*}
where we use the fact that $\|\Delta_{t-1}\|\leq M$ on event $\{\tau_M>t-1\}$.  Taking expectation again, if $\gamma_t\leq \epsilon^2/(4MK)$, then it follows that
\begin{align*}
\ev[\|\Delta_t\|^2I(\tau_M>t)]&\leq (1-2\gamma_t \epsilon^2 M^{-1}+4\gamma_t^2K)\ev[\|\Delta_{t-1}\|^2I(\tau_M>t-1)]+2\gamma_t^2K(1+2\|\theta^*\|^2)\\
&\leq (1-\gamma_t \epsilon^2 M^{-1})\ev[\|\Delta_{t-1}\|^2I(\tau_M>t-1)]+2\gamma_t^2K(1+2\|\theta^*\|^2).
\end{align*}
Applying Lemma \ref{lemma:su:zhu:10}, we conclude that, there exists a constant $K_M>0$ such that $ \ev[\|\Delta_{t}\|^2I(\tau_M>t)]\leq K_M\gamma_t$ for all $t\geq 0.$
\end{proof}

\begin{lemma}\label{lemma:s1s2:asymptotic:expansion}
Under Assumption \ref{Assumption:LH}, it follows that 
\begin{eqnarray*}
S_1+S_2=\frac{1}{T}(PGP)^-\sum_{t=2}^T \zeta_t+o_\p(T^{-1/2}),
\end{eqnarray*}
where $S_1$ and $S_2$ are defined in (\ref{eq:definition:s1:s2:s3}).
\end{lemma}
\begin{proof}[\textbf{\upshape Proof:}] 
Let $\hat{\theta}_0=\theta_0\in \mathbb{R}^p$ be the initial value for iteration. We define sequence
\begin{align*}
\hat{\theta}_t&=c+P(\hat{\theta}_{t-1}-\gamma_t h_t-c) \textrm{ with } h_t=G\hat{\theta}_{t-1}-G\theta^*+\zeta_t, \textrm{ for } t\geq 1,
\end{align*}
where $G\in \mathbb{R}^{p\times p}$ is the positive definite matrix defined in Assumption \ref{Assumption:LH}\ref{LH:4}, $\zeta_t$ is the process defined in (\ref{eq:iteration:proof}), and $c\in \mathbb{R}^p$ satisfies $Bc=b$. The proof is divided into four steps.

\noindent{{Step 1:}} This step is to show that $\lim_{t\to \infty}\hat{\theta}_t=\theta^*$ almost surely. Let us define $\hat{\Delta}_t=\hat{\theta}_t-\theta^*$, which is different from $\Delta_t=\theta_t-\theta^*$. By the fact that $P(\theta^*-c)=\theta^*-c$, we have
\begin{align}
\hat{\Delta}_t&=c+P(\hat{\theta}_{t-1}-\gamma_th_t-c)-\theta^*=c+P(\hat{\Delta}_{t-1}-\gamma_th_t+\theta^*-c)-\theta^*=P\hat{\Delta}_{t-1}-\gamma_t Ph_t\nonumber\\
&=P\hat{\Delta}_{t-1}-\gamma_t P (G\hat{\theta}_{t-1}-G\theta^*+\zeta_t)=P\hat{\Delta}_{t-1}-\gamma_t P G\hat{\Delta}_{t-1}-\gamma_t P\zeta_t =P(I-\gamma_t G)\hat{\Delta}_{t-1}-\gamma_t P\zeta_t \label{eq:lemma:bound:linear:model:eq1}.
\end{align}
As a consequence, it follows from (\ref{eq:lemma:bound:linear:model:eq1}) that
\begin{align}
\|\hat{\Delta}_t\|^2&=\|P(I-\gamma_t G)\hat{\Delta}_{t-1}\|^2-2\gamma_t \zeta_t^\top P(I-\gamma_t G)\hat{\Delta}_{t-1}+\gamma_t^2\|P\zeta_t\|^2\nonumber\\
&\leq (1-\gamma_t \lambda)^2 \|\hat{\Delta}_{t-1}\|^2-2\gamma_t \zeta_t^\top P(I-\gamma_t G)\hat{\Delta}_{t-1}+\gamma_t^2\|\zeta_t\|^2, \label{eq:lemma:bound:linear:model:eq1.2}
\end{align}
where $\lambda>0$ is the smallest eigenvalue of $G$. Taking conditional expectation, it follows that
\begin{align*}
\ev(\|\hat{\Delta}_t\|^2|\mcF_{t-1})&\leq (1-2\gamma_t\lambda+\gamma_t^2\lambda^2)\|\hat{\Delta}_{t-1}\|^2+\gamma_t^2\ev(\|\zeta_t\|^2|\mcF_{t-1})=(1+\gamma_t^2\lambda^2)\|\hat{\Delta}_{t-1}\|^2+\gamma_t^2\ev(\|\zeta_t\|^2|\mcF_{t-1})-2\gamma_t\lambda \|\hat{\Delta}_{t-1}\|^2\\
&\leq(1+\gamma_t^2\lambda^2)\|\hat{\Delta}_{t-1}\|^2+\gamma_t^2K(1+\|\theta_t\|^2)-2\gamma_t\lambda \|\hat{\Delta}_{t-1}\|^2\\
&\leq(1+\gamma_t^2\lambda^2)\|\hat{\Delta}_{t-1}\|^2+\gamma_t^2K(1+2\|\Delta_t\|^2+2\|\theta^*\|^2)-2\gamma_t\lambda \|\hat{\Delta}_{t-1}\|^2,
\end{align*}
where Lemma \ref{lemma:PR:equal:R}\ref{lemma:PR:equal:R:3} is used. Since $\lim_{t\to \infty}\|\Delta_t\|=0$ almost surely by Lemma \ref{lemma:xt:to:xstar} and $\sum_{t=1}^\infty \gamma_t^2<\infty$ by Assumption \ref{Assumption:LH}\ref{LH:1},  it follows that $\sum_{t=1}^\infty \gamma_t^2K(1+2\|\Delta_t\|^2+2\|\theta^*\|^2)<\infty$ almost surely. Hence, Robbins-Siegmund Theorem (e.g., see \cite{robbins1971convergence}) implies that
\begin{eqnarray*}
\lim_{t\to \infty}\|\hat{\Delta}_t\|^2\to \hat{V},\quad \; \sum_{t=1}^\infty \gamma_t\|\hat{\Delta}_t\|^2<\infty\; \textrm{ almost surely},
\end{eqnarray*}
for some random variable $\hat{V}$.
Since $\sum_{t=1}^\infty \gamma_t=\infty$, we conclude that  $\lim_{t\to \infty}\|\hat{\Delta}_t\|^2=0$ almost surely.

\noindent {Step 2:} Let us define stopping times  $\hat{\tau}_M=\inf\{j\geq 1: \|\hat{\Delta}_j\|>M\}$  and  $\tau_M=\inf\{j\geq 1: \|\Delta_j\|>M\}$ for $M>0$. This step is to prove that for any $M>0$, there exists  a constant $K_M>0$ such that
\begin{eqnarray}
\ev[\|\hat{\Delta}_t\|^2I(\hat{\tau}_M>t, \tau_M>t)]\leq K_M\gamma_t \; \textrm{ for all } t\geq 1. \label{eq:lemma:bound:linear:model:eq2}
\end{eqnarray} 
Using (\ref{eq:lemma:bound:linear:model:eq1.2}) again, we have
\begin{align*}
\|\hat{\Delta}_t\|^2I(\hat{\tau}_M>t, \tau_M>t)\leq&\; \|\hat{\Delta}_t\|^2I((\hat{\tau}_M>t-1,\tau_M>t-1)\\
\leq&\; (1-\gamma_t\lambda)^2\|\hat{\Delta}_{t-1}\|^2I(\hat{\tau}_M>t-1,\tau_M>t-1)+\gamma_t^2\|\zeta_t\|^2 I(\hat{\tau}_M>t-1,\tau_M>t-1)\\
&-2\gamma_t \zeta_t^\top P(I-\gamma_t G)\hat{\Delta}_{t-1} I(\hat{\tau}_M>t-1,\tau_M>t-1).
\end{align*}
Taking conditional expectation and noticing that $ \{\hat{\tau}_M>t-1,\tau_M>t-1\}\in \mcF_{t-1}$,  Lemma  \ref{lemma:PR:equal:R}\ref{lemma:PR:equal:R:3} further leads to
\begin{align*}
\ev[\|\hat{\Delta}_t\|^2I(\hat{\tau}_M>t,\tau_M>t)|\mcF_{t-1}]&\leq (1-\gamma_t\lambda)^2 \|\hat{\Delta}_{t-1}\|^2I(\hat{\tau}_M>t-1, \tau_M>t-1)+\gamma_t^2  \ev(\|\zeta_t\|^2|\mcF_{t-1})I(\hat{\tau}_M>t-1, \tau_M>t-1)\\
&\leq\bigg((1-\gamma_t\lambda)^2 \|\hat{\Delta}_{t-1}\|^2+\gamma_t^2 K(1+\|\theta_{t-1}\|^2)\bigg)I(\hat{\tau}_M>t-1,\tau_M>t-1)\\
&\leq \bigg((1-\gamma_t\lambda)^2 \|\hat{\Delta}_{t-1}\|^2+\gamma_t^2 K(1+2\|\Delta_{t-1}\|^2+2\|\theta^*\|^2)\bigg)I(\hat{\tau}_M>t-1,\tau_M>t-1)\\
&\leq (1-2\lambda \gamma_t+\lambda^2\gamma_t^2) \|\hat{\Delta}_{t-1}\|^2I(\hat{\tau}_M>t-1,\tau_M>t-1)+2\gamma_t^2K(1+M^2+\|\theta^*\|^2),
\end{align*}
where we use the fact that $\|\Delta_{t-1}\|\leq M$ when $\tau_M>t-1$.
Taking expectation again, we have
\begin{align*}
\ev[\|\hat{\Delta}_t\|^2I(\hat{\tau}_M>t, \tau_M>t)]&\leq (1-2\lambda \gamma_t+\lambda^2\gamma_t^2) \ev[\|\hat{\Delta}_{t-1}\|^2I(\hat{\tau}_M>t-1, \tau_M>t-1)]+2\gamma_t^2K(1+M^2+\|\theta^*\|^2)\\
&\leq  (1-\lambda \gamma_t)\ev[\|\hat{\Delta}_{t-1}\|^2I(\hat{\tau}_M>t-1,\tau_M>t-1)]+2\gamma_t^2K(1+M^2+\|\theta^*\|^2)\; \textrm{ when } \gamma_t\leq 1/\lambda,
\end{align*}
which further implies that
\begin{eqnarray*}
\frac{\ev[\|\hat{\Delta}_t\|^2I(\hat{\tau}_M>t, \tau_M>t)]}{\gamma_t}\leq  \frac{\gamma_{t-1}(1-\lambda \gamma_t)}{\gamma_t}\frac{\ev[\|\Delta_{t-1}\|^2I(\hat{\tau}_M>t-1, \tau_M>t-1)]}{\gamma_{t-1}}+2\gamma_tK(1+M^2+\|\theta^*\|^2).
\end{eqnarray*}
Now applying Lemma \ref{lemma:su:zhu:10}, we conclude that $\sup_{1\leq t <\infty}{\ev[\|\hat{\Delta}_t\|^2I(\hat{\tau}_M>t, \tau_M>t)]}/{\gamma_t}<\infty$,
which further implies (\ref{eq:lemma:bound:linear:model:eq2}).

\noindent{Step 3:} This step is to show 
\begin{eqnarray}
\frac{1}{\sqrt{T}}\sum_{t=2}^{T}\frac{\hat{\theta}_{t-1}-\hat{\theta}_t}{\gamma_t}=o_\p(1).\label{eq:lemma:bound:linear:model:step3}
\end{eqnarray}
Since both $\hat{\theta}_t$ and $\theta_t$ are strongly consistent by Step 1 and Lemma \ref{lemma:xt:to:xstar}, for any $\epsilon>0$, there exists a constant $M>0$ such that 
\begin{equation}
\pr(\sup_{1\leq t <\infty}\|\hat{\Delta}_t\|\leq M)\geq 1-\epsilon,\quad \pr(\sup_{1\leq t <\infty}\|{\Delta}_t\|\leq M)\geq 1-\epsilon. \label{eq:lemma:bound:linear:model:eq1.5}
\end{equation}
By direction examination, it follows that
\begin{align*}
\frac{1}{\sqrt{T}}\sum_{t=2}^{T}\frac{\hat{\theta}_{t-1}-\hat{\theta}_t}{\gamma_t}&=\frac{1}{\sqrt{T}}\sum_{t=2}^{T}\frac{\hat{\theta}_{t-1}-\theta^*+\theta^*-\hat{\theta}_t}{\gamma_t}=\frac{1}{\sqrt{T}}\sum_{t=2}^{T}\frac{\hat{\theta}_{t-1}-\theta^*}{\gamma_t}-\frac{1}{\sqrt{T}}\sum_{t=2}^{T}\frac{\hat{\theta}_t-\theta^*}{\gamma_t}=\frac{1}{\sqrt{T}}\sum_{t=1}^{T-1}\frac{\hat{\theta}_{t}-\theta^*}{\gamma_{t+1}}-\frac{1}{\sqrt{T}}\sum_{t=2}^{T}\frac{\hat{\theta}_t-\theta^*}{\gamma_t}\\
&=\frac{1}{\sqrt{T}}\frac{\hat{\theta}_1-\theta^*}{\gamma_2}-\frac{1}{\sqrt{T}}\frac{\hat{\theta}_T-\theta^*}{\gamma_T}+\frac{1}{\sqrt{T}}\sum_{t=2}^{T-1}(\hat{\theta}_t-\theta^*)(\gamma_{t+1}^{-1}-\gamma_t^{-1}):=D_1-D_2+D_3.
\end{align*}
It suffices to bound the three terms in right side of the last equation. Clearly $D_1=o_\p(1)$. For $D_2$, we have the following bound
\begin{align*}
\|D_2\|&=\frac{1}{\sqrt{T}\gamma_T}\|\hat{\theta}_T-\theta^*\|=\frac{1}{\sqrt{T}\gamma_T}\|\hat{\Delta}_T\|\\
&\leq \frac{1}{\sqrt{T}\gamma_T}\|\hat{\Delta}_T\|I(\hat{\tau}_M>T, \tau_M>T)+\frac{1}{\sqrt{T}\gamma_T}\|\hat{\Delta}_T\|I(\hat{\tau}_M\leq T)+\frac{1}{\sqrt{T}\gamma_T}\|\hat{\Delta}_T\|I(\tau_M\leq T):=D_{21}+D_{22}+D_{23}.
\end{align*}
By  (\ref{eq:lemma:bound:linear:model:eq2}) and Assumption \ref{Assumption:LH}\ref{LH:1}, we have
\begin{eqnarray*}
\ev(D_{21})\leq \frac{1}{\sqrt{T}\gamma_T}\sqrt{\ev[\|\Delta_T\|^2I(\hat{\tau_M}>T, \tau_M>T)]}\leq \sqrt{\frac{K_M}{T\gamma_T}}=\sqrt{\frac{K_M}{\gamma T^{1-\rho}}}\to 0.
\end{eqnarray*}
The definitions of $\hat{\tau}_M$ and $\tau_M$ indicate that $\{\sup_{1\leq t<\infty}\|\hat{\Delta}_t\|\leq M\}\subset \{\hat{\tau}_M>T\}$ and $\{\sup_{1\leq t<\infty}\|{\Delta}_t\|\leq M\}\subset \{{\tau}_M>T\}$. By (\ref{eq:lemma:bound:linear:model:eq1.5}), we see that 
\begin{eqnarray}
\pr(\hat{\tau}_M>T)\geq \pr(\sup_{1\leq t <\infty}\|\hat{\Delta}_t\|\leq M)\geq 1-\epsilon,\quad \pr(\tau_M>T)\geq \pr(\sup_{1\leq t <\infty}\|\Delta_t\|\leq M)\geq 1-\epsilon. \label{eq:lemma:bound:linear:model:eq3}
\end{eqnarray}
Since for any $\delta>0$, it follow that
\begin{eqnarray*}
D_{22}=\begin{cases}
\frac{1}{\sqrt{T}\gamma_T}\|\hat{\Delta}_T\| & \textrm{ if } \hat{\tau}_M\leq T;\\
0 & \textrm{ if } \hat{\tau}_M> T;\\
\end{cases}\;\;\;\; D_{23}=\begin{cases}
\frac{1}{\sqrt{T}\gamma_T}\|\hat{\Delta}_T\| & \textrm{ if } {\tau}_M\leq T;\\
0 & \textrm{ if } {\tau}_M> T,\\
\end{cases}
\end{eqnarray*}
we see that 
\begin{eqnarray}
\{D_{22}>\delta/3\}\subset \{\hat{\tau}_M\leq T\},\quad  \{D_{23}>\delta/3\}\subset \{{\tau}_M\leq T\}\label{eq:lemma:bound:linear:model:eq3.5}
\end{eqnarray}
Combining the above inequalities, for any $\delta>0$, we deduce that
\begin{align*}
\pr(\|D_2\|>\delta)&\leq \pr(D_{21}>\delta/3)+\pr(D_{22}>\delta/3)+\pr(D_{23}>\delta/3)\\
&\leq \frac{3}{\delta}\sqrt{\frac{K_M}{\gamma T^{1-\rho}}}+\pr(\hat{\tau}_M\leq T)+\pr({\tau}_M\leq T)\leq \frac{3}{\delta}\sqrt{\frac{K_M}{\gamma T^{1-\rho}}}+2\epsilon,
\end{align*}
which further implies that $\lim_{T\to \infty}\pr(\|D_2\|>\delta)\leq 2\epsilon.$ Since $\epsilon>0$ can be arbitrarily chosen, we show that $D_2=o_\p(1)$. To handle $D_3$, we use the following decomposition
\begin{align*}
\|D_3\|\leq&\; \frac{1}{\sqrt{T}}\sum_{t=2}^{T-1}\|\hat{\theta}_t-\theta^*\||\gamma_{t+1}^{-1}-\gamma_t^{-1}|I(\hat{\tau}_M>t, \tau_M>t)+\frac{1}{\sqrt{T}}\sum_{t=2}^{T-1}\|\hat{\theta}_t-\theta^*\||\gamma_{t+1}^{-1}-\gamma_t^{-1}|I(\hat{\tau}_M\leq t)\\
&+\frac{1}{\sqrt{T}}\sum_{t=2}^{T-1}\|\hat{\theta}_t-\theta^*\||\gamma_{t+1}^{-1}-\gamma_t^{-1}|I(\tau_M\leq t)\\
\leq&\;\frac{1}{\sqrt{T}}\sum_{t=2}^{T-1}\|\hat{\Delta}_t\||\gamma_{t+1}^{-1}-\gamma_t^{-1}|I(\hat{\tau}_M>t, \tau_M>t)+\frac{1}{\sqrt{T}}\sum_{t=2}^{T-1}\|\hat{\Delta}_t\||\gamma_{t+1}^{-1}-\gamma_t^{-1}|I(\hat{\tau}_M\leq T)\\
&+\frac{1}{\sqrt{T}}\sum_{t=2}^{T-1}\|\hat{\Delta}_t\||\gamma_{t+1}^{-1}-\gamma_t^{-1}|I({\tau}_M\leq T):= D_{31}+D_{32}+D_{33}.
\end{align*}
We obtain from (\ref{eq:lemma:bound:linear:model:eq2}) that
\begin{align*}
\ev\bigg(\sum_{t=2}^{\infty}\frac{1}{\sqrt{t}}\|\hat{\Delta}_t\||\gamma_{t+1}^{-1}-\gamma_t^{-1}|I(\hat{\tau}_M>t, \tau_M>t)\bigg)&\leq  \sum_{t=2}^{\infty}\frac{|\gamma_{t+1}^{-1}-\gamma_t^{-1}|}{\sqrt{t}}\ev[\|\hat{\Delta}_t\|I(\hat{\tau}_M>t, \tau_M>t)]\\
&\leq \sum_{t=2}^{\infty}\frac{|\gamma_{t+1}^{-1}-\gamma_t^{-1}|}{\sqrt{t}}\sqrt{\ev[\|\hat{\Delta}_t\|^2I(\hat{\tau}_M>t, \tau_M>t)]}\leq \sum_{t=2}^{\infty}\frac{|\gamma_{t+1}^{-1}-\gamma_t^{-1}|}{\sqrt{t}}\sqrt{K_M\gamma_t}\\
&= \sqrt{K_M}\sum_{t=2}^{\infty}\frac{1}{\sqrt{t}}\frac{1}{\gamma}|(t+1)^\rho-t^\rho|\sqrt{\gamma t^{-\rho}}=\sqrt{\gamma K_M}\sum_{t=2}^{\infty}\frac{1}{\sqrt{t}}t^{\rho}\bigg[\bigg(\frac{t+1}{t}\bigg)^\rho-1\bigg]\sqrt{ t^{-\rho}}\\
&\leq \sqrt{\gamma K_M}\sum_{t=2}^{\infty}\frac{1}{\sqrt{t}}t^{\rho}\bigg[\bigg(\frac{t+1}{t}\bigg)-1\bigg]\sqrt{ t^{-\rho}}= \sqrt{\gamma K_M}\sum_{t=2}^{\infty}\frac{1}{t^{3/2-\rho/2}}<\infty,
\end{align*}
where we use Assumption \ref{Assumption:LH}\ref{LH:1} that $\gamma_t=\gamma t^{\rho}$ for some $\rho \in (1/2, 1)$. The above inequality also implies that
\begin{eqnarray*}
\sum_{t=2}^{\infty}\frac{1}{\sqrt{t}}\|\hat{\Delta}_t\||\gamma_{t+1}^{-1}-\gamma_t^{-1}|I(\hat{\tau}_M>t, \tau_M>t)<\infty\; \textrm{ almost surely.}
\end{eqnarray*}
As a consequence of Kronecker's lemma, we show that $D_{31}=o_\p(1)$. Using (\ref{eq:lemma:bound:linear:model:eq3}) and similar arguments as (\ref{eq:lemma:bound:linear:model:eq3.5}), for any $\delta>0$, we have
\begin{align*}
\pr(\|D_3\|>\delta)&\leq \pr(D_{31}>\delta/3)+\pr(D_{32}>\delta/3)+\pr(D_{33}>\delta/3)\\
&\leq \pr(D_{31}>\delta/3)+\pr(\hat{\tau}_M\leq T)+\pr({\tau}_M\leq T)\leq\pr(D_{31}>\delta/3)+2\epsilon.
\end{align*}
Taking limit, it holds that $\lim_{T\to \infty}\pr(\|D_3\|>\delta)\leq 2\epsilon$. Since $\epsilon>0$ can be arbitrarily chosen, we show that $D_3=o_\p(1)$.  Combining the rates of $D_1, D_2, D_3$, we verify (\ref{eq:lemma:bound:linear:model:step3})

\noindent{Step 4:}  Using (\ref{eq:lemma:bound:linear:model:eq1}), we have
\begin{eqnarray*}
\hat{\Delta}_t=P(I-\gamma_t G)\hat{\Delta}_{t-1}-\gamma_t P\zeta_t=P\hat{\Delta}_{t-1}-\gamma_t PG\hat{\Delta}_{t-1}-\gamma_t P\zeta_t.
\end{eqnarray*}
Since $P\hat{\Delta}_t=\hat{\Delta}_t$ for $t\geq 1$, it further leads to
\begin{eqnarray*}
\gamma_t PGP\hat{\Delta}_{t-1}=\gamma_t PG\hat{\Delta}_{t-1}=-P(\hat{\Delta}_t-\hat{\Delta}_{t-1})-\gamma_tP\zeta_t=-P(\hat{\theta}_t-\hat{\theta}_{t-1})-\gamma_tP\zeta_t \quad \textrm{ for all } t\geq 2.
\end{eqnarray*}
Taking summation, we show that
\begin{eqnarray*}
\sum_{t=2}^{T+1} PGP\hat{\Delta}_{t-1}=PGP\hat{\Delta}_{T}+\sum_{t=2}^{T} PGP\hat{\Delta}_{t-1}=PGP\hat{\Delta}_{T}-P\sum_{t=2}^T\frac{\theta_t-\theta_{t-1}}{\gamma_t}-P\sum_{t=2}^T \zeta_t,
\end{eqnarray*}
which further implies that
\begin{align*}
\frac{1}{\sqrt{T}}PGP\sum_{t=1}^T\hat{\Delta}_{t}=\frac{1}{\sqrt{T}}PGP\hat{\Delta}_T-\frac{1}{\sqrt{T}}P\sum_{t=2}^T\frac{\theta_t-\theta_{t-1}}{\gamma_t}-\frac{1}{\sqrt{T}}P\sum_{t=2}^T\zeta_t:=E_1-E_2-E_3.
\end{align*}
Using the strong consistency of $\hat{\theta}_t$ in Step 1 and (\ref{eq:lemma:bound:linear:model:step3}) in Step 3, we show that $E_1=o_\p(1)$ and $E_2=o_\p(1)$. Moreover, by iterative substitution and the fact that $\hat{\theta}_0=\theta_0$, (\ref{eq:lemma:bound:linear:model:eq1}) leads to
\begin{align}
\hat{\Delta}_t &=P(I-\gamma_t G)\hat{\Delta}_{t-1}-\gamma_t P\zeta_t =\bigg[\prod_{j=1}^t P(I-\gamma_j G)\bigg]\hat{\Delta}_0+\sum_{j=1}^t \bigg[\prod_{i=j+1}^t P(I-\gamma_i G)\bigg]\gamma_j P\zeta_j \nonumber\\
&=\bigg[\prod_{j=1}^t P(I-\gamma_j G)\bigg]{\Delta}_0+\sum_{j=1}^t \bigg[\prod_{i=j+1}^t P(I-\gamma_i G)P\bigg]\gamma_j P\zeta_j,\nonumber
\end{align}
which, by averaging, further implies that
\begin{eqnarray*}
\frac{1}{T}\sum_{t=1}^T\hat{\Delta}_t=\frac{1}{T}\sum_{t=1}^T \bigg[\prod_{j=1}^t P(I-\gamma_j A)\bigg]{\Delta}_0+\frac{1}{T}\sum_{t=1}^T\sum_{j=1}^t \bigg[\prod_{i=j+1}^t P(I-\gamma_i G)P\bigg]\gamma_j P\zeta_j=S_1+S_2.
\end{eqnarray*}
Notice that $(PGP)^-P=(PGP)^-$ by Lemma \ref{lemma:matrices:same:eigenvector}, we complete the proof. 
\end{proof}

\begin{lemma}\label{lemma:bound:3}
Under Assumption \ref{Assumption:LH}, it follows that $S_3=o_\p(T^{-1/2})$, where $S_3$ is defined in (\ref{eq:definition:s1:s2:s3}).
\end{lemma}
\begin{proof}[\textbf{\upshape Proof:}] 
Changing the order of summation leads to
\begin{align*}
S_3=\frac{1}{T}\sum_{t=1}^T\sum_{j=1}^t \bigg[\prod_{i=j+1}^t P(I-\gamma_i G)P\bigg]\gamma_j P(R(\theta_{j-1})-G\Delta_{j-1})=\frac{1}{T}\sum_{j=1}^T\sum_{t=j}^T \bigg[\prod_{i=j+1}^t P(I-\gamma_i G)P\bigg]\gamma_j P(R(\theta_{j-1})-G\Delta_{j-1}).
\end{align*}
By Lemma \ref{lemma:xt:to:xstar}, for any $\epsilon>0$, there exists a constant $M>0$ such that
\begin{eqnarray}
\pr(\tau_M>T)\geq \pr(\sup_{1\leq t <\infty}\|\Delta_t\|\leq M)\geq 1-\epsilon, \label{eq:lemma:asymptotic:expansion:eq1}
\end{eqnarray} 
where  $\tau_M=\inf\{j\geq 1: \|\Delta_j\|>M\}$ is the stopping time defined in Lemma \ref{lemma:convergence:rate:xt}. Setting $\alpha_j^T=\gamma_j \sum_{t=j}^T \left(\prod_{i=j+1}^t P(I-\gamma_i G)P\right)$, Lemmas \ref{lemma:lemma:matrices:iteration:my} and \ref{lemma:V:compact:range} lead to
\begin{align*}
\|\sqrt{T}S_3\|&\leq \bigg\|\frac{1}{\sqrt{T}}\sum_{j=1}^{T}\alpha_j^T P(R(\theta_{j-1})-G\Delta_{j-1})\bigg\|\leq\frac{1}{\sqrt{T}}\sum_{j=1}^{T}\|\alpha_j^T\| \|P(R(\theta_{j-1})-G\Delta_{j-1})\|\leq  \frac{K}{\sqrt{T}}\sum_{j=1}^{T}\|R(\theta_{j-1})-G\Delta_{j-1}\|\\
&\leq \frac{K^2}{\sqrt{T}}\sum_{j=1}^{T}\|\Delta_{j-1}\|^2\leq \frac{K^2}{\sqrt{T}}\sum_{j=1}^{T}\|\Delta_{j-1}\|^2I(\tau_M\leq j-1)+ \frac{K^2}{\sqrt{T}}\sum_{j=1}^{T}\|\Delta_{j-1}\|^2I(\tau_M>j-1)\\
&\leq \frac{K^2}{\sqrt{T}}\sum_{j=1}^{T}\|\Delta_{j-1}\|^2I(\sup_{1\leq t<\infty}\|\Delta_t\|>  M)+ \frac{K^2}{\sqrt{T}}\sum_{j=1}^{T}\|\Delta_{j-1}\|^2I(\tau_M>j-1):=S_{31}+S_{32}.
\end{align*}
For the first term, using (\ref{eq:lemma:asymptotic:expansion:eq1}) and similar arguments as (\ref{eq:lemma:bound:linear:model:eq3.5}), for any $\delta>0$, we have
\begin{eqnarray*}
\pr(S_{31}>\delta/2)\leq \pr(\sup_{1\leq t <\infty}\|\Delta_t\|>M)\leq \epsilon.
\end{eqnarray*}
For the second term,  Lemma \ref{lemma:convergence:rate:xt} implies that
\begin{eqnarray*}
\ev\bigg(\sum_{j=1}^\infty \frac{\|\Delta_j\|^2I(\tau_M>j)}{j^{1/2}}\bigg)\leq K_M\sum_{j=1}^\infty \frac{\gamma_j}{j^{1/2}}= K_M\sum_{j=1}^\infty \frac{\gamma}{j^{1/2+\rho}}<\infty,
\end{eqnarray*}
where we use Assumption \ref{Assumption:LH}\ref{LH:1} that $\gamma_j=\gamma j^{-\rho}$ for some $\rho \in (1/2, 1)$. The above inequality also implies that $\pr(\sum_{j=1}^\infty j^{-1/2}{\|\Delta_j\|^2I(\tau_M>j)}<\infty)=1$.
Applying Kronecker's lemma, we show that $S_{32}\to 0$ almost surely as $T \to \infty$. Combining the bounds of $S_{31}$ and $S_{32}$, we conclude that
\begin{eqnarray*}
\lim_{T\to \infty}\pr(\|\sqrt{T}S_3\|>\delta)\leq \lim_{T\to \infty}\pr(S_{31}>\delta/2)+\lim_{T\to \infty}\pr(S_{32}>\delta/2)\leq \epsilon.
\end{eqnarray*}
Since $\epsilon>0$ can be arbitrarily chosen, we show that $\sqrt{T}S_3=o_\p(1)$.
\end{proof}

\begin{lemma}\label{lemma:verify:clt}
Under Assumption \ref{Assumption:LH}, it follows that $T^{-1/2}\sum_{t=1}^T \zeta_t\cid N(0, S)$.
\end{lemma}
\begin{proof}[\textbf{\upshape Proof:}] 
We decompose the process $\zeta_t$ as follows:
\begin{align*}
\zeta_t&=\nabla l(\theta_{t-1}, Z_t)-\nabla L(\theta_{t-1})=\nabla l(\theta^*, Z_t)+[\nabla l(\theta_{t-1}, Z_t)-\nabla l(\theta^*, Z_t)-\nabla L(\theta_{t-1})+\nabla L(\theta^*)]:=\eta_t+\xi_t.
\end{align*}
%By Lemma \ref{lemma:PR:equal:R} and Lemma \ref{lemma:xt:to:xstar}, we conclude that
%\begin{eqnarray*}
%\sup_{1\leq t<\infty}\ev(\|\xi_t\|^2 |\mcF_{t-1})\leq K(1+\sup_{1\leq t<\infty}\|\theta_{t-1}\|^2)<\infty,
%\end{eqnarray*}
 Assumption \ref{Assumption:LH}\ref{LH:6} and Lemma \ref{lemma:xt:to:xstar} imply that
\begin{eqnarray*}
\ev(\|\xi_t\|^2 |\mcF_{t-1})\leq \ev(\|\nabla l(\theta_{t-1}, Z_t)-\nabla l(\theta^*, Z_t)\|^2|\mcF_{t-1})\leq \delta(\|\theta_{t-1}-\theta^*\|)\to 0\; \textrm{ almost surely.}
\end{eqnarray*}
Moreover, by Cauchy–Schwarz inequality, it follows that
\begin{eqnarray*}
\ev(\|\eta_t\xi_t^\top\| |\mcF_{t-1})\leq \ev(\|\eta_t\|\|\xi_t\| |\mcF_{t-1})\leq \sqrt{\ev(\|\eta_t\|^2)}\sqrt{\ev(\|\xi_t\|^2|\mcF_{t-1})}\to 0\; \textrm{ almost surely.}
\end{eqnarray*}
As a consequence of the above two inequalities, we show that
\begin{eqnarray*}
\ev(\zeta_t\zeta_t^\top |\mcF_{t-1})=\ev(\eta_t\eta_t^\top)+2\ev(\eta_t\xi_t^\top|\mcF_{t-1})+\ev(\xi_t\xi_t^\top|\mcF_{t-1})\to S\; \textrm{ almost surely},
\end{eqnarray*}
where $S=\ev[\nabla l(\theta^*, Z) \nabla l^\top (\theta^*, Z)]\in \mathbb{R}^{p\times p}$ is a positive definite matrix defined in Assumption \ref{Assumption:LH}\ref{LH:5}. For any $\epsilon>0$, direct calculation leads to
\begin{align*}
&\;\;\;\;\ev(\|\zeta_t\|^2I(\|\zeta_t\|>\epsilon \sqrt{T})|\mcF_{t-1})\leq\ev[2(\|\eta_t\|^2+\|\xi_t\|^2)I(\|\eta_t\|+\|\xi_t\|>\epsilon \sqrt{T})|\mcF_{t-1}]\\
&=\ev[2(\|\eta_t\|^2+\|\xi_t\|^2)I(\|\eta_t\|+\|\xi_t\|>\epsilon \sqrt{T}, \|\eta_t\|\geq \|\xi_t\|)|\mcF_{t-1}]+\ev[2(\|\eta_t\|^2+\|\xi_t\|^2)I(\|\eta_t\|+\|\xi_t\|>\epsilon \sqrt{T}, \|\eta_t\|< \|\xi_t\|)|\mcF_{t-1}]\\
&\leq 4\ev[\|\eta_t\|^2I(\|\eta_t\|\geq \epsilon \sqrt{T}/2)|\mcF_{t-1}]+4\ev[\|\xi_t\|^2I(\|\xi_t\|\geq \epsilon \sqrt{T}/2)|\mcF_{t-1}]\leq4\ev[\|\eta_t\|^2I(\|\eta_t\|\geq \epsilon \sqrt{T}/2)]+4\delta(\|\theta_{t-1}-\theta^*\|).
\end{align*}
Since $\theta_j=\nabla l(\theta^*, Z_j)$ are i.i.d., and $\lim_{t\to \infty}\delta(\|\theta_{t-1}-\theta^*\|)=0$ almost surely, we conclude that
\begin{eqnarray*}
\lim_{T\to \infty}\frac{1}{T}\sum_{t=1}^{T}\ev(\|\zeta_t\|^2I(\|\zeta_t\|>\epsilon \sqrt{T})|\mcF_{t-1})=0\; \textrm{ almost surely}.
\end{eqnarray*}
By the C.L.T. for martingale-difference arrays (e.g., see  \cite{Pollard1984convergence}), we prove the asymptotic normality.
\end{proof}

\section{Proof of Lemma  \ref{lemma:more:efficient:theta:P}}\label{section:proof:lemma:1}
It suffices to show that $G^{-1}-(PGP)^-$ is positive semidefinite and has rank $p-d$. Since rank$(P)=d$ by Assumption \ref{Assumption:LH}\ref{LH:7}, and $P$ is diagonalisable, there exists an orthogonal matrix $U\in \mathbb{R}^{p\times p}$ such that 
\begin{eqnarray*}
P=U\begin{pmatrix}
I_d&0\\
0&0
\end{pmatrix}U^\top,\quad   U^\top GU=\begin{pmatrix}
X & Y\\
Y^\top & Z
\end{pmatrix},
\end{eqnarray*}
for some matrices $X, Y, Z$ with comfortable dimensions. As a consequence, it follows that
\begin{eqnarray*}
PGP=U\begin{pmatrix}
I_d&0\\
0&0
\end{pmatrix} U^\top G  U\begin{pmatrix}
I_d&0\\
0&0
\end{pmatrix} U^\top=U\begin{pmatrix}
I_d&0\\
0&0
\end{pmatrix} \begin{pmatrix}
X & Y\\
Y^\top & Z
\end{pmatrix}\begin{pmatrix}
I_d&0\\
0&0
\end{pmatrix} U^\top =U\begin{pmatrix}
X & 0\\
0 &0
\end{pmatrix}U^\top.
\end{eqnarray*}
Let $S=Z-Y^\top X^{-1}Y\in \mathbb{R}^{(p-d)\times (p-d)}$ be the Schur complement of $X$. Since $G$ is positive definite by Assumption \ref{Assumption:LH}\ref{LH:4}, so is $S$. The matrix block inversion formula implies that
\begin{align*}
G^{-1}&=U\begin{pmatrix}
X & Y\\
Y^\top & Z
\end{pmatrix}^{-1}U^\top=U\begin{pmatrix}
X^{-1}+X^{-1}YS^{-1}Y^\top X^{-1}&-X^{-1}YS^{-1}\\
-S^{-1}Y^\top X^{-1}& S^{-1}
\end{pmatrix}U^\top\\
&=U\begin{pmatrix}
X^{-1}&0\\
0& 0
\end{pmatrix}U^\top+U\begin{pmatrix}
X^{-1}YS^{-1}Y^\top X^{-1}&-X^{-1}YS^{-1}\\
-S^{-1}Y^\top X^{-1}& S^{-1}\\
\end{pmatrix}U^\top=(PGP)^-+U\begin{pmatrix}
X^{-1}YS^{-1/2}\\
-S^{-1/2}
\end{pmatrix}\begin{pmatrix}
S^{-1/2}Y^\top X^{-1} , -S^{-1/2}
\end{pmatrix}U^\top,
\end{align*}
which proves the positive semidefiniteness.  Because $\textrm{rank}(S^{-1/2})=p-d$ and  $(
S^{-1/2}Y^\top X^{-1} , -S^{-1/2}
)\in \mathbb{R}^{(p-d)\times p}$, we verify that $G^{-1}-(PGP)^-$ has rank $p-d$.

\section{Proof of Lemma  \ref{lemma:estimation:variance} }\label{section:proof:lemma:2}

\begin{lemma}\label{lemma:davis:kahan:theorem}
Let $\Sigma, \hat{\Sigma}\in \mathbb{R}^{p\times p}$ be symmetric with eigenvalues $\lambda_1\geq \ldots \lambda_p$ and $\hat{\lambda}_1\geq \ldots \hat{\lambda}_p$.  Fixing $1\leq r\leq s\leq p$, let us define $d=s-r+1$, and let  $V=(v_r,v_{r+1},\ldots, v_s)\in \mathbb{R}^{p\times d}$,  $\hat{V}=(\hat{v}_r,\hat{v}_{r+1},\ldots, \hat{v}_s)\in \mathbb{R}^{p\times d}$ have orthonormal columns satisfying $\Sigma v_j=\lambda_j v_j$ and $\hat{\Sigma}\hat{v}_j=\hat{\lambda}_jv_j$ for $j\in \{r,r+1,\ldots, s\}$. If $e:=\inf\{|\hat{\lambda}-\lambda|: \lambda \in [\lambda_s, \lambda_r], \hat{\lambda}\in (-\infty, \hat{\lambda}_{s-1}]\cup [\hat{\lambda}_{r+1},\infty)\}>0$, where $\hat{\lambda}_0:=-\infty$ and $\hat{\lambda}_{p+1}:=\infty$, then it follows that $\|VV^\top-\hat{V}\hat{V}^\top\|\leq 2\|\hat{\Sigma}-\Sigma\|/e$.
Moreover, the eigenvalues satisfies $|\hat{\lambda}_i-\lambda_i|\leq \|\hat{\Sigma}-\Sigma\|.$
\end{lemma}
\begin{proof}[\textbf{\upshape Proof:}] 
It follows from Davis-Kahan Theorem (e.g., see \cite{yu2015useful}) and Weyl's inequality \cite{weyl1912asymptotische}.
\end{proof}

\begin{lemma}\label{lemma:converge:pesudo:inverse:An:to:A}
Let $\Sigma, \hat{\Sigma}_n\in \mathbb{R}^{p\times p}$ be positive semidefinite matrices such that $\textrm{rank}(\hat{\Sigma}_n)=\textrm{rank}(\Sigma)$ and  $\hat{\Sigma}_n \to \Sigma$ as $n\to \infty$. Then $\lim_{n \to \infty}\hat{\Sigma}_n^- =\Sigma^-$.
\end{lemma}
\begin{proof}[\textbf{\upshape Proof:}] 
Let distinct eigenvalues of $\Sigma$ be $\rho_1>\rho_2>\cdots>\rho_d=0$, and suppose that there are $k_j\geq 1$ eigenvalues $\lambda_{j,1}=\lambda_{j,2}=\cdots=\lambda_{j,k_j}$ equal to $\rho_j$, for $j\in \{1,\ldots, d\}$. We denote $v_{j, s}$ as the eigenvector corresponding to eigenvalue $\lambda_{j, s}$. Similarly, we define $(\hat{\lambda}_{j, s}, \hat{v}_{j,s})$ as the eigenpair of $\hat{\Sigma}_n$ for $j\in \{1,\ldots, d\}$ and $s\in \{1,\ldots, k_j\}$. However, in general, we do not have  $\hat{\lambda}_{j,1}=\hat{\lambda}_{j,2}=\cdots=\hat{\lambda}_{j,k_j}$. Moreover, the eigenvalues can be chosen to be in an increasing order such that
\begin{align*}
\hat{\lambda}_{j, 1}\geq  \hat{\lambda}_{j, 1}\geq \cdots \geq \hat{\lambda}_{j, k_j}\quad  &\textrm{ for all } j\in \{1,\ldots, d\},\\
\hat{\lambda}_{1, s_1}\geq \hat{\lambda}_{2, s_2} \geq \cdots \geq \hat{\lambda}_{d, s_d} \quad &\textrm{ for all } s_j\in \{1,\ldots, k_j\} \textrm{ and } j\in \{1,\ldots, d\}.
\end{align*}
By Lemma \ref{lemma:davis:kahan:theorem}, we see that $\hat{\lambda}_{j, s}\to \lambda_{j, s}=\rho_j$ for all $j\in \{1,\ldots, d\}$. As a consequence, when $n$ is sufficiently large, there exists a constant $\epsilon>0$ such that 
\begin{eqnarray*}
\rho_{j+1}<\rho_{j}-\epsilon \leq \hat{\lambda}_{j, s}\leq \rho_{j}+\epsilon<\rho_{j-1} \quad \textrm{ for all } s\in \{1,\ldots, k_j\} \textrm{ and } j\in \{1,\ldots, d-1\}.
\end{eqnarray*}
Since $\textrm{rank}(\Sigma_n)=\textrm{rank}(\Sigma)$, it holds that $\hat{\lambda}_{d, s}=\lambda_{d, s}=\rho_d=0$. For each $j\in \{1,\ldots, d-1\}$, applying Lemma \ref{lemma:davis:kahan:theorem} to eigenpairs $(\lambda_{j,s}, v_{j,s})$ and $(\hat{\lambda}_{j,s}, \hat{v}_{j,s})$ with $s\in \{1,\ldots, k_j\}$, we have $e\geq \epsilon$ and
\begin{eqnarray*}
\bigg\|\sum_{s=1}^{k_j}\hat{v}_{j,s}\hat{v}_{j,s}^\top-\sum_{s=1}^{k_j}{v}_{j,s}{v}_{j,s}^\top\bigg\|\leq 2\|\hat{\Sigma}_n-\Sigma\|/\epsilon=o_\p(1),
\end{eqnarray*}
which  futher implies that
\begin{align*}
\bigg\|\sum_{s=1}^{k_j}\hat{\lambda}_{j,s}^{-1}\hat{v}_{j,s}\hat{v}_{j,s}^\top-\sum_{s=1}^{k_j}\lambda_{j, s}^{-1}{v}_{j,s}{v}_{j,s}^\top\bigg\|&=\bigg\|\sum_{s=1}^{k_j}\hat{\lambda}_{j,s}^{-1}\hat{v}_{j,s}\hat{v}_{j,s}^\top-\sum_{s=1}^{k_j}\rho_{j}^{-1}{v}_{j,s}{v}_{j,s}^\top\bigg\|\\
&\leq \bigg\|\sum_{s=1}^{k_j}\hat{\lambda}_{j,s}^{-1}\hat{v}_{j,s}\hat{v}_{j,s}^\top-\sum_{s=1}^{k_j}\rho_j^{-1}\hat{v}_{j,s}\hat{v}_{j,s}^\top\bigg\|+\bigg\|\sum_{s=1}^{k_j}\rho_j^{-1}\hat{v}_{j,s}\hat{v}_{j,s}^\top-\sum_{s=1}^{k_j}\rho_{j}^{-1}{v}_{j,s}{v}_{j,s}^\top\bigg\|\\
&\leq\sum_{s=1}^{k_j}|\hat{\lambda}_{j, s}^{-1}-\rho_j^{-1}|\|\hat{v}_{j,s}\hat{v}_{j,s}^\top\|+\rho_j^{-1}\bigg\|\sum_{s=1}^{k_j}\hat{v}_{j,s}\hat{v}_{j,s}^\top-\sum_{s=1}^{k_j}{v}_{j,s}{v}_{j,s}^\top\bigg\|=\sum_{s=1}^{k_j}|\hat{\lambda}_{j, s}^{-1}-\rho_j^{-1}|+o_\p(1)=o_\p(1),
\end{align*}
where we used the fact that $\rho_j>0$ for $j\in \{1,\ldots, d-1\}$. Finally, notice that
\begin{align*}
\Sigma^-=\sum_{j=1}^{d-1}\sum_{s=1}^{k_j}\lambda_{j, s}^{-1}{v}_{j,s}{v}_{j,s}^\top, \quad \hat{\Sigma}^-=\sum_{j=1}^{d-1}\sum_{s=1}^{k_j}\hat{\lambda}_{j,s}^{-1}\hat{v}_{j,s}\hat{v}_{j,s}^\top,
\end{align*}
we complete the proof.
\end{proof}

\begin{lemma}\label{lemma:convergence:pesudo:inverse}
Suppose a sequence of matrices $\{A_n\}_{n=1}^\infty\in \mathbb{R}^{p\times p}$ satisfies $\lim_{n\to \infty}A_n=A$ where $A\in \mathbb{R}^{p\times p}$ is positive definite. Let $P\in \mathbb{R}^{p\times p}$ be a projection matrix such that $P^2=P$ and $P^\top=P$.   Then $\lim_{n\to \infty}(PA_nP)^-=(PAP)^-.$
\end{lemma}
\begin{proof}[\textbf{\upshape Proof:}] 
Since $A$ is positive definite, so is $A_n$ when $n$ is sufficiently large. Hence $PA_nP$ and $PAP$ both have the same rank as $P$. The desired result follows from Lemma \ref{lemma:converge:pesudo:inverse:An:to:A}.
\end{proof}

\begin{lemma}\label{lemma:estimation:matrix}
Under Assumptions \ref{Assumption:LH} and \ref{Assumption:var}, it follows that $\hat{G}_{T}=G+o_\p(1)$, $(P\hat{G}_{T}P)^-=(P{G}P)^-+o_\p(1)$, and $\hat{S}_{T}=S+o_\p(1)$ .
\end{lemma}
\begin{proof}[\textbf{\upshape Proof:}] 
Since $\overline{\theta}_T\to \theta^*$ almost surely as $T\to \infty$ by Lemma \ref{lemma:xt:to:xstar}, it follows from the continuity of $\theta \to \nabla^2 l(\theta, Z)$ at $\theta^*$ in Assumption \ref{Assumption:var} that $\lim_{T\to \infty}\|\nabla^2 l(\overline{\theta}_T, Z_T)-\nabla^2 l(\theta^*, Z_T)\|=0$ almost surely. As a consequence, when $T\to \infty$, we have
\begin{eqnarray*}
\bigg\|\frac{1}{T}\sum_{t=1}^T \bigg(\nabla^2 l(\overline{\theta}_t, Z_t)-\nabla^2 l(\theta^*, Z_t)\bigg)\bigg\|\leq \frac{1}{T}\sum_{t=1}^T \bigg\|\nabla^2 l(\overline{\theta}_t, Z_t)-\nabla^2 l(\theta^*, Z_t)\bigg\|\to 0 \textrm{ almost surely}.
\end{eqnarray*}
By  Assumption \ref{Assumption:var}, Lebesgue's Dominated Convergence Theorem, and L.L.N., we can see
\begin{eqnarray*}
\frac{1}{T}\sum_{t=1}^T \nabla^2 l(\theta^*, Z_t)=\ev[\nabla^2 l(\theta^*, Z_t)]+o_\p(1)=\nabla^2 L(\theta^*)+o_\p(1).
\end{eqnarray*}
Combining the above, we show that $\hat{G}_{T}=G+o_\p(1)$.

Similarly, we have $\lim_{T\to \infty}\|\nabla l(\overline{\theta}_T, Z_T)\nabla l^\top(\overline{\theta}_T, Z_T)-\nabla l(\theta^*, Z_T)\nabla l^\top(\theta^*, Z_T)\|=0$ almost surely by the differentiability of $\theta \to \nabla l(\theta, Z)$ in Assumption \ref{Assumption:var}. Moreover, by L.L.N., we can derive $\hat{S}_{T}=S+o_\p(1)$. Finally, applying Lemma \ref{lemma:convergence:pesudo:inverse}, we complete the proof.
\end{proof}
Lemma  \ref{lemma:estimation:variance}  is a direct consequence of Lemma \ref{lemma:estimation:matrix}.

\section{Proof of Theorem \ref{theorem:specification:test}}\label{section:proof:theorem:2}
Under $H_0$, by Theorem \ref{theorem:asymptotic:expansion}, it follows that
\begin{eqnarray*}
\overline{\theta}_{T, P}-\theta^*=-\frac{1}{T}\sum_{t=1}^T (PGP)^-\zeta_t+o_\p(T^{-1/2}),\quad \overline{\theta}_{T,I}-\theta^*=-\frac{1}{T}\sum_{t=1}^T G^{-1}\zeta_t+o_\p(T^{-1/2}).
\end{eqnarray*}
Since $P(PGP)^-=(PGP)^-$ by Lemma \ref{lemma:matrices:same:eigenvector}, we have
\begin{eqnarray*}
(I-P)(\overline{\theta}_{T, P}-\overline{\theta}_{T, I})=\frac{1}{T}\sum_{t=1}^T (I-P)[G^{-1}-(PGP)^-]\zeta_t+o_\p(T^{-1/2})=\frac{1}{T}\sum_{t=1}^T (I-P)G^{-1}\zeta_t+o_\p(T^{-1/2}).
\end{eqnarray*}
By Lemma \ref{lemma:verify:clt}, we show that $\sqrt{T}(\overline{\theta}_{T, P}-\overline{\theta}_{T, I})\cid N(0, W),$
where $W=(I-P)G^{-1}SG^{-1}(I-P)$. By delta method, we have $\sqrt{T}[W^-]^{1/2}(\overline{\theta}_{T, P}-\overline{\theta}_{T, I})\cid N(0, V)$, where $V=\textrm{Diag}(1,\ldots, 1,0,\ldots,0)\in \mathbb{R}^{p\times p}$ is a squared matrix with rank $p-d$.
%\begin{eqnarray*}
%\sqrt{T}[W^-]^{1/2}(\overline{\theta}_{T, P}-\overline{\theta}_{T, I})\cid N(0, V), \textrm{ where } V=\textrm{Diag}(\underbrace{1,\ldots, 1}_{p-d},0,\ldots,0)\in \mathbb{R}^{p\times p}.
%\end{eqnarray*}
The above convergence further leads to $T(\overline{\theta}_{T, P}-\overline{\theta}_{T, I})^\top W^-(\overline{\theta}_{T, P}-\overline{\theta}_{T, I})\cid \chi^2(p-d).$ By Lemma \ref{lemma:estimation:matrix}, it follows that $\hat{G}_{T,I}=G+o_\p(1)$ and $\hat{S}_{T,I}=S+o_\p(1)$. Moreover, both $W$ and $\hat{W}$ are of rank $p-d$. As a consequence of Lemma \ref{lemma:converge:pesudo:inverse:An:to:A}, it follows $\hat{W}=W+o_\p(1)$. Applying Slutsky's Theorem, we compete the proof of the result under $H_0$.

Under $H_1$, since $B\theta^*=b+\beta$, for some $\beta\neq 0$. Consider the following decomposition $\theta^*=\tilde{\theta}^*+\mu$
with $B\tilde{\theta}^*=b$ and $B\mu=\beta$. Clearly, $(I-P)\mu \neq 0$, as $(I-P)\mu=0$ implies $P\mu=\mu$ and $\mu\in \textrm{Ker}(B)$, which is impossible. Since $B\overline{\theta}_{T,P}=B\tilde{\theta}^*=b$, we have
\begin{align}
(I-P)(\overline{\theta}_{T, P}-\overline{\theta}_{T, I})&=(I-P)(\overline{\theta}_{T, P}-\theta^*+\theta^*-\overline{\theta}_{T, I})\nonumber=(I-P)(\overline{\theta}_{T, P}-\tilde{\theta}^*-\mu+\theta^*-\overline{\theta}_{T, I})\nonumber\\
&=-(I-P)\mu-(I-P)(\overline{\theta}_{T, I}-\theta^*).\label{eq:theorem:specification:test:eq1}
\end{align}
Moreover, by Lemma \ref{lemma:matrices:same:eigenvector}, we have $\hat{W}^{-}(I-P)=(I-P)\hat{W}^{-}=W^-$. 
Following  (\ref{eq:theorem:specification:test:eq1}), we have
\begin{align*}
T(\overline{\theta}_{T, P}-\overline{\theta}_{T, I})^\top \hat{W}^-(\overline{\theta}_{T, P}-\overline{\theta}_{T, I})=T\mu^\top \hat{W}^- \mu+T(\overline{\theta}_{T, I}-\theta^*)^\top \hat{W}^- (\overline{\theta}_{T, I}-\theta^*)+2T(\overline{\theta}_{T, I}-\theta^*)^\top \hat{W}^-\mu:=J_1+J_2+J_3.
\end{align*}
For $S_1$, let $\hat{\lambda}_1, \hat{\lambda}_{p-d}$ and $\lambda_1, \lambda_{p-d}$ be the largest and smallest non-zero eigenvalues of $\hat{W}$ and $W$ respectively. By Lemma \ref{lemma:converge:pesudo:inverse:An:to:A}, we know $\hat{\lambda}_1\leq 2\lambda_1$ and $\hat{\lambda}_{p-d}\geq \lambda_{p-d}/2$ with probability approaching 1. Then by Lemma \ref{lemma:matrices:same:eigenvector}, we conclude that
\begin{eqnarray*}
J_1\geq \frac{T}{\hat{\lambda}}\|(I-P)\mu\|^2\geq \frac{T}{2\lambda}\|(I-P)\mu\|^2, \textrm{ with probability approaching 1.}
\end{eqnarray*}
Since Theorem \ref{theorem:asymptotic:expansion} implies that $\overline{\theta}_{T, I}-\theta^*=O_\p(T^{-1/2})$, it follows that
\begin{eqnarray*}
J_2\leq T\|W^-\|\|\overline{\theta}_{T, I}-\theta^*\|^2\leq \frac{T}{\hat{\lambda}_{p-d}}\|\overline{\theta}_{T, I}-\theta^*\|^2\leq \frac{2T}{{\lambda}_{p-d}}\|\overline{\theta}_{T, I}-\theta^*\|^2=O_\p(1).
\end{eqnarray*}
Similarly, by Cauchy–Schwarz inequality, we can show
\begin{eqnarray*}
|J_3|\leq 2T\|\hat{W}^-\|\|\overline{\theta}_{T, I}-\theta^*\|\|\mu\|\leq  \frac{2T}{\hat{\lambda}_{p-d}}\|\overline{\theta}_{T, I}-\theta^*\|\|\mu\|\leq  \frac{4T}{{\lambda}_{p-d}}\|\overline{\theta}_{T, I}-\theta^*\|\|\mu\|
=O_\p(T^{1/2}).
\end{eqnarray*}
Combining the three bounds, we prove that $T(\overline{\theta}_{T, P}-\overline{\theta}_{T, I})^\top \hat{W}^-(\overline{\theta}_{T, P}-\overline{\theta}_{T, I})\to \infty$ with probability approaching 1.

Suppose the local alternative $H_a: B\theta^*=b+\beta/\sqrt{T}$ holds.  Consider the following decomposition $\theta^*=\tilde{\theta}^*+\mu/\sqrt{T}$
with $B\tilde{\theta}^*=b$ and $B\mu=\beta$.
By Lemma \ref{lemma:matrices:same:eigenvector}, we have $(\hat{W}^{-})^{1/2}(I-P)=(I-P)(\hat{W}^{-})^{1/2}=(\hat{W}^{-})^{1/2}$.
By similar proof to  (\ref{eq:theorem:specification:test:eq1}), we have 
\begin{eqnarray*}
(I-P)(\overline{\theta}_{T, P}-\overline{\theta}_{T, I})= -(I-P)\mu/\sqrt{T}-(I-P)(\overline{\theta}_{T, I}-\theta^*),
\end{eqnarray*}
which further leads to
\begin{align*}
(\hat{W}^-)^{1/2}(\overline{\theta}_{T, P}-\overline{\theta}_{T, I})=(\hat{W}^-)^{1/2}(I-P)(\overline{\theta}_{T, P}-\overline{\theta}_{T, I})=-(\hat{W}^-)^{1/2}(I-P)\mu/\sqrt{T}-(\hat{W}^-)^{1/2}(\overline{\theta}_{T,I}-\theta^*):=R_1-R_2.
\end{align*}
Since $\hat{W}^{-}=W+o_\p(1)$, it follows that $\sqrt{T}R_1=-({W}^-)^{1/2}(I-P)\mu+o_\p(1)=-({W}^-)^{1/2}\mu+o_\p(1).$ Moreover,  Theorem \ref{theorem:asymptotic:expansion} implies that
\begin{eqnarray*}
\sqrt{T}R_2=({W}^-)^{1/2}(\overline{\theta}_{T,I}-\theta^*)+o_\p(1)\cid N(0, ({W}^-)^{1/2}G^{-1}SG^{-1}({W}^-)^{1/2}).
\end{eqnarray*}
By direct calculation, it can be verified that
\begin{align*}
({W}^-)^{1/2}G^{-1}SG^{-1}({W}^-)^{1/2}=({W}^-)^{1/2}(I-P)G^{-1}SG^{-1}(I-P)({W}^-)^{1/2}&=({W}^-)^{1/2}W({W}^-)^{1/2}\\
&=\textrm{Diag}(\underbrace{1,\ldots, 1}_{p-d},0,\ldots,0)\in \mathbb{R}^{p\times p}.
\end{align*}
As a consequence, we show that $T(\overline{\theta}_{T, P}-\overline{\theta}_{T, I})^\top \hat{W}^-(\overline{\theta}_{T, P}-\overline{\theta}_{T, I})\cid \chi^2(\mu^\top W^- \mu , p-d).$

%\section*{References}
%There are several ways to include references into the article. For example, using Bibtex and the file "trial" has been created so that you can see how one can do. To refer to the content of the file use for instance
%\cite{andersson,anderson1993totally,lauritzen1996graphical,gratzer2002general} or \cite{drton2008iterative}. At the end of this template "trial.bib" is included and below there is also an example of how to include references in an alternative way.
%\section*{References}

% To ensure accuracy, get them from MathSciNet whenever possible. Typeset them with BibTeX using JMVA's style file, \texttt{myjmva.bst}.
%\bibliography{}
\bibliographystyle{myjmva}
%\begin{thebibliography}
%\bibliography{trial}
%\end{document}

\bibliography{ref}{}

\begin{thebibliography}{23}
\expandafter\ifx\csname natexlab\endcsname\relax\def\natexlab#1{#1}\fi
\providecommand{\bibinfo}[2]{#2}
\ifx\xfnm\relax \def\xfnm[#1]{\unskip,\space#1}\fi
%Type = Inproceedings
\bibitem[{Bottou(1991)}]{bottou1991stochastic}
\bibinfo{author}{L.~Bottou}, \bibinfo{title}{Stochastic gradient learning in
  neural networks}, in: \bibinfo{booktitle}{Proceedings of Neuro-N\^imes 91},
  \bibinfo{publisher}{EC2}, \bibinfo{address}{Nimes, France},
  \bibinfo{year}{1991}.
%Type = Article
\bibitem[{Bottou et~al.(2018)Bottou, Curtis and
  Nocedal}]{bottou2018optimization}
\bibinfo{author}{L.~Bottou}, \bibinfo{author}{F.~E. Curtis},
  \bibinfo{author}{J.~Nocedal}, \bibinfo{title}{Optimization methods for
  large-scale machine learning}, \bibinfo{journal}{Siam Review}
  \bibinfo{volume}{60} (\bibinfo{year}{2018}) \bibinfo{pages}{223--311}.
%Type = Article
\bibitem[{Chen et~al.(2020)Chen, Lee, Tong and Zhang}]{chen2020statistical}
\bibinfo{author}{X.~Chen}, \bibinfo{author}{J.~D. Lee}, \bibinfo{author}{X.~T.
  Tong}, \bibinfo{author}{Y.~Zhang}, \bibinfo{title}{Statistical inference for
  model parameters in stochastic gradient descent}, \bibinfo{journal}{Annals of
  Statistics} \bibinfo{volume}{48} (\bibinfo{year}{2020})
  \bibinfo{pages}{251--273}.
%Type = Misc
\bibitem[{Dua and Graff(2019)}]{Dua:2019}
\bibinfo{author}{D.~Dua}, \bibinfo{author}{C.~Graff}, \bibinfo{title}{{UCI}
  machine learning repository}, \bibinfo{year}{2019}. \bibinfo{note}{University
  of California, Irvine, School of Information and Computer Science:
  http://archive.ics.uci.edu/ml}.
%Type = Article
\bibitem[{Fang et~al.(2018)Fang, Xu and Yang}]{fang2018online}
\bibinfo{author}{Y.~Fang}, \bibinfo{author}{J.~Xu}, \bibinfo{author}{L.~Yang},
  \bibinfo{title}{Online bootstrap confidence intervals for the stochastic
  gradient descent estimator}, \bibinfo{journal}{Journal of Machine Learning
  Research} \bibinfo{volume}{19} (\bibinfo{year}{2018})
  \bibinfo{pages}{3053--3073}.
%Type = Inproceedings
\bibitem[{Gemulla et~al.(2011)Gemulla, Nijkamp, Haas and
  Sismanis}]{rainer2011matrix}
\bibinfo{author}{R.~Gemulla}, \bibinfo{author}{E.~Nijkamp},
  \bibinfo{author}{P.~J. Haas}, \bibinfo{author}{Y.~Sismanis},
  \bibinfo{title}{Large-scale matrix factorization with distributed stochastic
  gradient descent}, KDD '11, \bibinfo{publisher}{Association for Computing
  Machinery}, \bibinfo{address}{New York}, \bibinfo{year}{2011}, pp.
  \bibinfo{pages}{69--77}.
%Type = Article
\bibitem[{Godichon-Baggioni and Portier(2017)}]{godichon2017averaged}
\bibinfo{author}{A.~Godichon-Baggioni}, \bibinfo{author}{B.~Portier},
  \bibinfo{title}{An averaged projected robbins-monro algorithm for estimating
  the parameters of a truncated spherical distribution},
  \bibinfo{journal}{Electronic Journal of Statistics} \bibinfo{volume}{11}
  (\bibinfo{year}{2017}) \bibinfo{pages}{1890--1927}.
%Type = Article
\bibitem[{Gorman and Hero(1990)}]{gorman1990lower}
\bibinfo{author}{J.~D. Gorman}, \bibinfo{author}{A.~O. Hero},
  \bibinfo{title}{Lower bounds for parametric estimation with constraints},
  \bibinfo{journal}{IEEE Transactions on Information Theory}
  \bibinfo{volume}{36} (\bibinfo{year}{1990}) \bibinfo{pages}{1285--1301}.
%Type = Inproceedings
\bibitem[{Gower et~al.(2019)Gower, Loizou, Qian, Sailanbayev, Shulgin and
  Richt{\'a}rik}]{gower2019sgd}
\bibinfo{author}{R.~M. Gower}, \bibinfo{author}{N.~Loizou},
  \bibinfo{author}{X.~Qian}, \bibinfo{author}{A.~Sailanbayev},
  \bibinfo{author}{E.~Shulgin}, \bibinfo{author}{P.~Richt{\'a}rik},
  \bibinfo{title}{{SGD}: General analysis and improved rates},
  volume~\bibinfo{volume}{97} of \text{\bibinfo{series}{Proceedings of Machine
  Learning Research}}, \bibinfo{publisher}{PMLR}, \bibinfo{year}{2019}, pp.
  \bibinfo{pages}{5200--5209}.
%Type = Article
\bibitem[{J{\'e}r{\^o}me(2005)}]{jerome2005central}
\bibinfo{author}{L.~J{\'e}r{\^o}me}, \bibinfo{title}{A central limit theorem
  for robbins monro algorithms with projections}  (\bibinfo{year}{2005}).
  \bibinfo{note}{Preprint on webpage at
  \url{https://cermics.enpc.fr/cermics-rapports-recherche/2005/CERMICS-2005/CERMICS-2005-285.pdf}.
  Accessed on 03.20.2022.}
%Type = Article
\bibitem[{Moore et~al.(2008)Moore, Sadler and Kozick}]{moore2008maximum}
\bibinfo{author}{T.~J. Moore}, \bibinfo{author}{B.~M. Sadler},
  \bibinfo{author}{R.~J. Kozick}, \bibinfo{title}{Maximum-likelihood
  estimation, the cram{\'e}r--rao bound, and the method of scoring with
  parameter constraints}, \bibinfo{journal}{IEEE Transactions on Signal
  Processing} \bibinfo{volume}{56} (\bibinfo{year}{2008})
  \bibinfo{pages}{895--908}.
%Type = Article
\bibitem[{Nemirovski et~al.(2009)Nemirovski, Juditsky, Lan and
  Shapiro}]{nemirovski2009robust}
\bibinfo{author}{A.~Nemirovski}, \bibinfo{author}{A.~Juditsky},
  \bibinfo{author}{G.~Lan}, \bibinfo{author}{A.~Shapiro},
  \bibinfo{title}{Robust stochastic approximation approach to stochastic
  programming}, \bibinfo{journal}{SIAM Journal on Optimization}
  \bibinfo{volume}{19} (\bibinfo{year}{2009}) \bibinfo{pages}{1574--1609}.
%Type = Article
\bibitem[{Pelletier(2000)}]{pelletier2000asymptotic}
\bibinfo{author}{M.~Pelletier}, \bibinfo{title}{Asymptotic almost sure
  efficiency of averaged stochastic algorithms}, \bibinfo{journal}{SIAM Journal
  on Control and Optimization} \bibinfo{volume}{39} (\bibinfo{year}{2000})
  \bibinfo{pages}{49--72}.
%Type = Book
\bibitem[{Pollard(1984)}]{Pollard1984convergence}
\bibinfo{author}{D.~Pollard}, \bibinfo{title}{Convergence of Stochastic
  Processes}, \bibinfo{publisher}{Springer-Verlag}, \bibinfo{address}{Berlin,
  Heidelberg}, \bibinfo{year}{1984}.
%Type = Article
\bibitem[{Polyak(1990)}]{polyak1990new}
\bibinfo{author}{B.~T. Polyak}, \bibinfo{title}{New method of stochastic
  approximation type}, \bibinfo{journal}{Automation and remote control}
  \bibinfo{volume}{51} (\bibinfo{year}{1990}) \bibinfo{pages}{937--946}.
%Type = Article
\bibitem[{Polyak and Juditsky(1992)}]{polyak1992acceleration}
\bibinfo{author}{B.~T. Polyak}, \bibinfo{author}{A.~B. Juditsky},
  \bibinfo{title}{Acceleration of stochastic approximation by averaging},
  \bibinfo{journal}{SIAM journal on Control and Optimization}
  \bibinfo{volume}{30} (\bibinfo{year}{1992}) \bibinfo{pages}{838--855}.
%Type = Article
\bibitem[{Robbins and Monro(1951)}]{robbins1951stochastic}
\bibinfo{author}{H.~Robbins}, \bibinfo{author}{S.~Monro}, \bibinfo{title}{A
  stochastic approximation method}, \bibinfo{journal}{Annals of Mathematical
  Statistics} \bibinfo{volume}{22} (\bibinfo{year}{1951})
  \bibinfo{pages}{400--407}.
%Type = Incollection
\bibitem[{Robbins and Siegmund(1971)}]{robbins1971convergence}
\bibinfo{author}{H.~Robbins}, \bibinfo{author}{D.~Siegmund}, \bibinfo{title}{A
  convergence theorem for non negative almost supermartingales and some
  applications}, in: \bibinfo{editor}{J.~S. Rustagi} (Ed.),
  \bibinfo{booktitle}{Optimizing Methods in Statistics},
  \bibinfo{publisher}{Academic Press}, \bibinfo{year}{1971}, pp.
  \bibinfo{pages}{233--257}.
%Type = Incollection
\bibitem[{Ruppert(1991)}]{ruppert1988efficient}
\bibinfo{author}{D.~Ruppert}, \bibinfo{title}{Stochastic approximation}, in:
  \bibinfo{editor}{B.~K. Ghosh}, \bibinfo{editor}{P.~K. Sen} (Eds.),
  \bibinfo{booktitle}{Handbook of Sequential Analysis},
  \bibinfo{publisher}{Marcel Dekker}, \bibinfo{address}{New York},
  \bibinfo{year}{1991}, pp. \bibinfo{pages}{503--529.}
%Type = Article
\bibitem[{Su and Zhu(2018)}]{su2018uncertainty}
\bibinfo{author}{W.~J. Su}, \bibinfo{author}{Y.~Zhu},
  \bibinfo{title}{Uncertainty quantification for online learning and stochastic
  approximation via hierarchical incremental gradient descent},
  \bibinfo{journal}{arXiv preprint arXiv:1802.04876}  (\bibinfo{year}{2018}).
%Type = Article
\bibitem[{Weyl(1912)}]{weyl1912asymptotische}
\bibinfo{author}{H.~Weyl}, \bibinfo{title}{Das asymptotische verteilungsgesetz
  der eigenwerte linearer partieller differentialgleichungen (mit einer
  anwendung auf die theorie der hohlraumstrahlung)},
  \bibinfo{journal}{Mathematische Annalen} \bibinfo{volume}{71}
  (\bibinfo{year}{1912}) \bibinfo{pages}{441--479}.
%Type = Article
\bibitem[{Yu et~al.(2015)Yu, Wang and Samworth}]{yu2015useful}
\bibinfo{author}{Y.~Yu}, \bibinfo{author}{T.~Wang}, \bibinfo{author}{R.~J.
  Samworth}, \bibinfo{title}{A useful variant of the davis--kahan theorem for
  statisticians}, \bibinfo{journal}{Biometrika} \bibinfo{volume}{102}
  (\bibinfo{year}{2015}) \bibinfo{pages}{315--323}.
%Type = Inproceedings
\bibitem[{Zhang(2004)}]{zhang2004linearprediction}
\bibinfo{author}{T.~Zhang}, \bibinfo{title}{Solving large scale linear
  prediction problems using stochastic gradient descent algorithms},
  Proceedings of International Conference on Machine Learning,
  \bibinfo{publisher}{Association for Computing Machinery},
  \bibinfo{address}{New York}, \bibinfo{year}{2004}, pp.
  \bibinfo{pages}{919--926,}.

\end{thebibliography}
\clearpage
%\bibliographystyle{myjmva}
%\section*{}
%or one can use (see the template for details)
%\begin{thebibliography}{99}
%
%\bibitem[{Agresti(2013)}]{Agresti13}
%\bibinfo{author}{A.~Agresti}, \bibinfo{title}{{Categorical Data Analysis}},
%  \bibinfo{publisher}{Wiley, Hoboken}, \bibinfo{year}{2013}.
%%Type = Article
%\bibitem[{Aitchison and Silvey(1958)}]{aitchison1958maximum}
%\bibinfo{author}{J.~Aitchison}, \bibinfo{author}{S.D.~Silvey},
%  \bibinfo{title}{Maximum-likelihood estimation of parameters subject to
%  restraints}, \bibinfo{journal}{The Annals of Mathematical Statistics}
%  \bibinfo{volume}{29} (\bibinfo{year}{1958}) \bibinfo{pages}{813--828}.
% \bibitem{Balsu} A. Balsubramani, S. Dasgupta, Y. Freund, \newblock The fast
%convergence of incremental PCA, \newblock Advances in Neural Information
%Processing Systems 26 (2013) 3174--3182.
%
%\end{thebibliography}

\setcounter{subsection}{0}
\renewcommand{\thesubsection}{S.\arabic{subsection}}
\setcounter{section}{0}
\renewcommand{\thesection}{S.\arabic{section}}
\setcounter{subsubsection}{0}
\renewcommand{\thesubsubsection}{\textbf{S.\arabic{subsection}.\arabic{subsubsection}}}
\setcounter{equation}{0}
\renewcommand{\theequation}{S.\arabic{equation}}
\setcounter{lemma}{0}
\renewcommand{\thelemma}{S.\arabic{lemma}}
\setcounter{proposition}{0}
\renewcommand{\theproposition}{S.\arabic{proposition}}
\setcounter{Assumption}{0}
\renewcommand{\theAssumption}{C\arabic{Assumption}}
\setcounter{page}{1}

\begin{center}
 {\large	\textbf{Supplementary Material for ``Online Statistical Inference for Parameters Estimation with Linear-Equality Constraints"}}
\end{center}
\section*{}
This supplementary material summarizes several simulation results and an application to a real-world dataset.

\section{Estimation error and coverage probability}
\noindent DGP 1 (Linear Regression): Consider the model $Y=\sum_{j=1}^4\beta_j X_j+\epsilon$, with the true parameters $\theta^*=(\beta_1, \beta_2,\beta_3, \beta_4)^\top=(1.5,-3,2,1)^\top$. The covariates $X_1, X_2, X_3, X_4 \sim N(0, 1)$ and the error term $\epsilon\sim N(0, 9)$ are independent. The linear-equality constraint $\beta_2+\beta_3+\beta_4=0$ is used.

\noindent  DGP 2 (Logistic Regression): We  generate the model $P(Y=y|X_1, X_2, X_3, X_4)=[1+e^{-y\sum_{j=1}^4\beta_j X_j}]^{-1}$ for $y\in \{-1, 1\}$, with the true parameters $\theta^*=(\beta_1, \beta_2,\beta_3, \beta_4)^\top=(1,-2,-2,1.5)^\top$. The covariates $X_1, X_2, X_3, X_4$ follow the same distributions as DGP 1. The linear-equality constraint $\beta_2-\beta_3=0$ is applied.

We consider the APSGD estimate $\hat{\theta}_{T, P}=(\hat{\beta}_{1, P}, \hat{\beta}_{2, P}, \hat{\beta}_{3, P}, \hat{\beta}_{4, P})^\top$ using the proper projection matrix $P$ and the ASGD estimate $\hat{\theta}_{T, I}=(\hat{\beta}_{1, I}, \hat{\beta}_{2, I}, \hat{\beta}_{3, I}, \hat{\beta}_{4, I})^\top$ using identity projection matrix. The estimation error is evaluated by $|\hat{\beta}_{j, P}-\beta_j|$ and $|\hat{\beta}_{j, I}-\beta_j|$ over 500 runs. Moreover, during each run, we construct a 95\% level confidence interval for $\beta_j$, and examine whether the true $\beta_j$ is in the confidence interval or not. The learning rate is selected as $t^{-0.505}$.
The estimation error and coverage probability are reported in Tables \ref{table:rmse}-\ref{table:cp}.  First, from Table \ref{table:rmse}, we see that  the estimation errors of $\hat{\theta}_{T, P}$ and $\hat{\theta}_{T, I}$ decrease when the sample size $T$ increases. Second, for both linear and logistic models, the estimation error of $\hat{\theta}_{T, P}$ is uniformly smaller than $\hat{\theta}_{T, I}$ regardless of the sample size.  Third, Table \ref{table:cp} reveals that the coverage probabilities of the 95\% confidence intervals for $\beta_1, \ldots, \beta_4$ are around 95\%, which confirms the validity of our theoretical results.

\begin{table}[t!]
 \caption{Estimation error of $\hat{\theta}_{T, P}$ and $\hat{\theta}_{T, I}$.}
\label{table:rmse}
\centering
\begin{tabular}{ccccccc}
\hline \hline
                           &                & \multicolumn{2}{c}{Linear}                            &  & \multicolumn{2}{c}{Logistic}                          \\ \cline{1-4} \cline{6-7} 
   \rule{0pt}{13pt}                               & $T$            & $\hat{\theta}_{T, P}$ & $\hat{\theta}_{T, I}$ &  & $\hat{\theta}_{T, P}$ & $\hat{\theta}_{T, I}$ \\ \cline{1-4} \cline{6-7} 
\multirow{4}{*}{$\beta_1$} & $10^4$         & 0.0080                    & 0.0080                    &  & 0.0099                    & 0.0101                    \\
                           & $2\times 10^4$ & 0.0060                    & 0.0077                    &  & 0.0119                    & 0.0131                    \\
                           & $5\times 10^4$ & 0.0058                    & 0.0072                    &  & 0.0119                    & 0.0128                    \\
                           & $10^5$         & 0.0062                    & 0.0077                    &  & 0.0113                    & 0.0115                    \\ \cline{1-4} \cline{6-7} 
\multirow{4}{*}{$\beta_2$} & $10^4$         & 0.0055                    & 0.0055                    &  & 0.0065                    & 0.0066                    \\
                           & $2\times 10^4$ & 0.0043                    & 0.0057                    &  & 0.0076                    & 0.0084                    \\
                           & $5\times 10^4$ & 0.0043                    & 0.0048                    &  & 0.0076                    & 0.0085                    \\
                           & $10^5$         & 0.0045                    & 0.0054                    &  & 0.0073                    & 0.0073                    \\ \cline{1-4} \cline{6-7} 
\multirow{4}{*}{$\beta_3$} & $10^4$         & 0.0035                    & 0.0036                    &  & 0.0038                    & 0.0038                    \\
                           & $2\times 10^4$ & 0.0028                    & 0.0035                    &  & 0.0049                    & 0.0053                    \\
                           & $5\times 10^4$ & 0.0025                    & 0.0031                    &  & 0.0049                    & 0.0055                    \\
                           & $10^5$         & 0.0029                    & 0.0034                    &  & 0.0048                    & 0.0048                    \\ \cline{1-4} \cline{6-7} 
\multirow{4}{*}{$\beta_4$} & $10^4$         & 0.0023                    & 0.0023                    &  & 0.0028                    & 0.0028                    \\
                           & $2\times 10^4$ & 0.0020                    & 0.0023                    &  & 0.0035                    & 0.0040                    \\
                           & $5\times 10^4$ & 0.0021                    & 0.0026                    &  & 0.0035                    & 0.0039                    \\
                           & $10^5$         & 0.0020                    & 0.0024                    &  & 0.0032                    & 0.0032                   \\\hline\hline
\end{tabular}
\end{table}

% \usepackage{multirow}

\begin{table}[t!]
\caption{Coverage probability  of $\hat{\theta}_{T, P}$.}
\label{table:cp}
\centering
\begin{tabular}{cccccccccc}
\hline \hline
          & \multicolumn{4}{c}{Linear Model}                  &  & \multicolumn{4}{c}{Logistic Model}                \\ \cline{2-5} \cline{7-10} 
$T$       & $10^5$ & $2\times 10^5$ & $5\times 10^5$ & $10^6$ &  & $10^5$ & $2\times 10^5$ & $5\times 10^5$ & $10^6$ \\ \cline{1-5} \cline{7-10} 
$\beta_1$ & 0.942  & 0.934          & 0.944          & 0.966  &  & 0.918  & 0.938          & 0.954          & 0.956  \\
$\beta_2$ & 0.946  & 0.932          & 0.938          & 0.942  &  & 0.924  & 0.924          & 0.932          & 0.948  \\
$\beta_3$ & 0.960  & 0.938          & 0.968          & 0.944  &  & 0.924  & 0.924          & 0.932          & 0.945  \\
$\beta_4$ & 0.938  & 0.960          & 0.962          & 0.962  &  & 0.934  & 0.936          & 0.938          & 0.962  \\ \hline \hline
\end{tabular}
\end{table}

\section{Size and power}
To examine the empirical performance of the specification test  in (\ref{eq:test:procedure}), we modify the settings of DGP 1 and DGP 2 as follows. 

\noindent DGP 1 (Linear Regression): The coefficients are chosen to be $\theta^*=(\beta_1,\beta_2, \beta_3, \beta_4)^\top=(1.5,-3,2,1+r)^\top$, with $r=0, 0.005, 0.01, 0.015, 0.02, 0.025$. The hypothesis to be tested is $H_0: \beta_2+\beta_3+\beta_4=0$.

\noindent  DGP 2 (Logistic Regression): The coefficients are chosen to be $\theta^*=(\beta_1,\beta_2, \beta_3, \beta_4)^\top=(3,-2,-2+r,1)^\top$, with $r=0, 0.005, 0.01, 0.015, 0.02, 0.025$. The null hypothesis to be tested is $H_0: \beta_2-\beta_3=0$.

The scalar $r$ measures the level of model misspecification. When $r=0$, the model is correctly specified by the linear-equality constraint. We repeat the experiment 500 times with significance level $\alpha=0.05$ for different choices of $r$ and $T$, and the average rejection probabilities are reported in Figure \ref{figure:test}. First, Figure \ref{figure:test} indicates that the probabilities of rejecting the null hypothesis are around 0.95 when $r=0$, which suggests the proposed specification test has an correct asymptotic size ($\alpha=0.05$). Second, for different sample sizes, the rejection probability is monotonically increasing with respect to $r$. In particular, the specification test almost 100\% rejects $H_0$ when $r=0.02, 0.025$ for both linear and logistic models, which confirms the consistency of the specification test.

\begin{figure}[htp!]
\centering
\includegraphics[width=2 in]{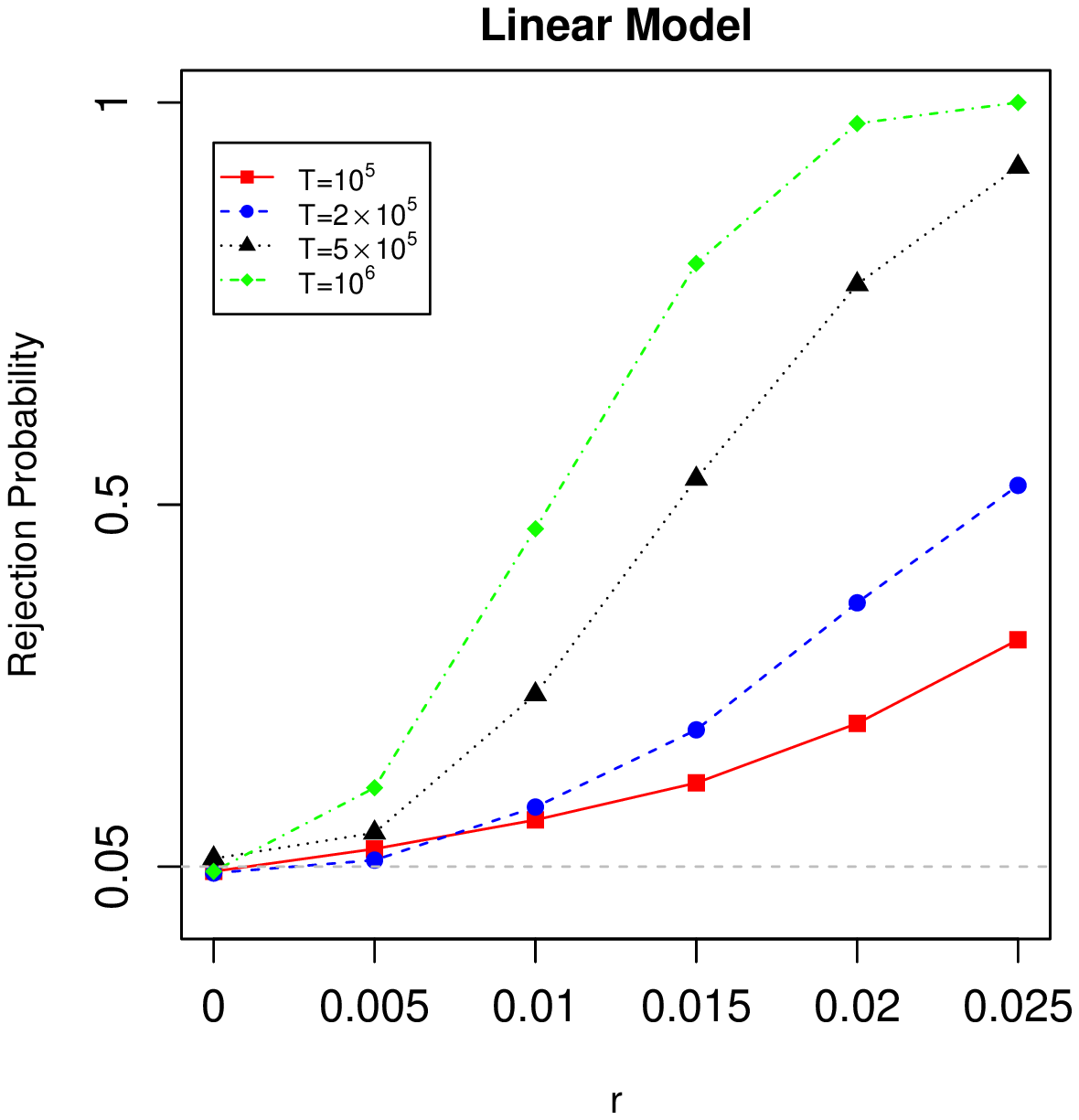}\hspace*{2cm}
\includegraphics[width=2 in]{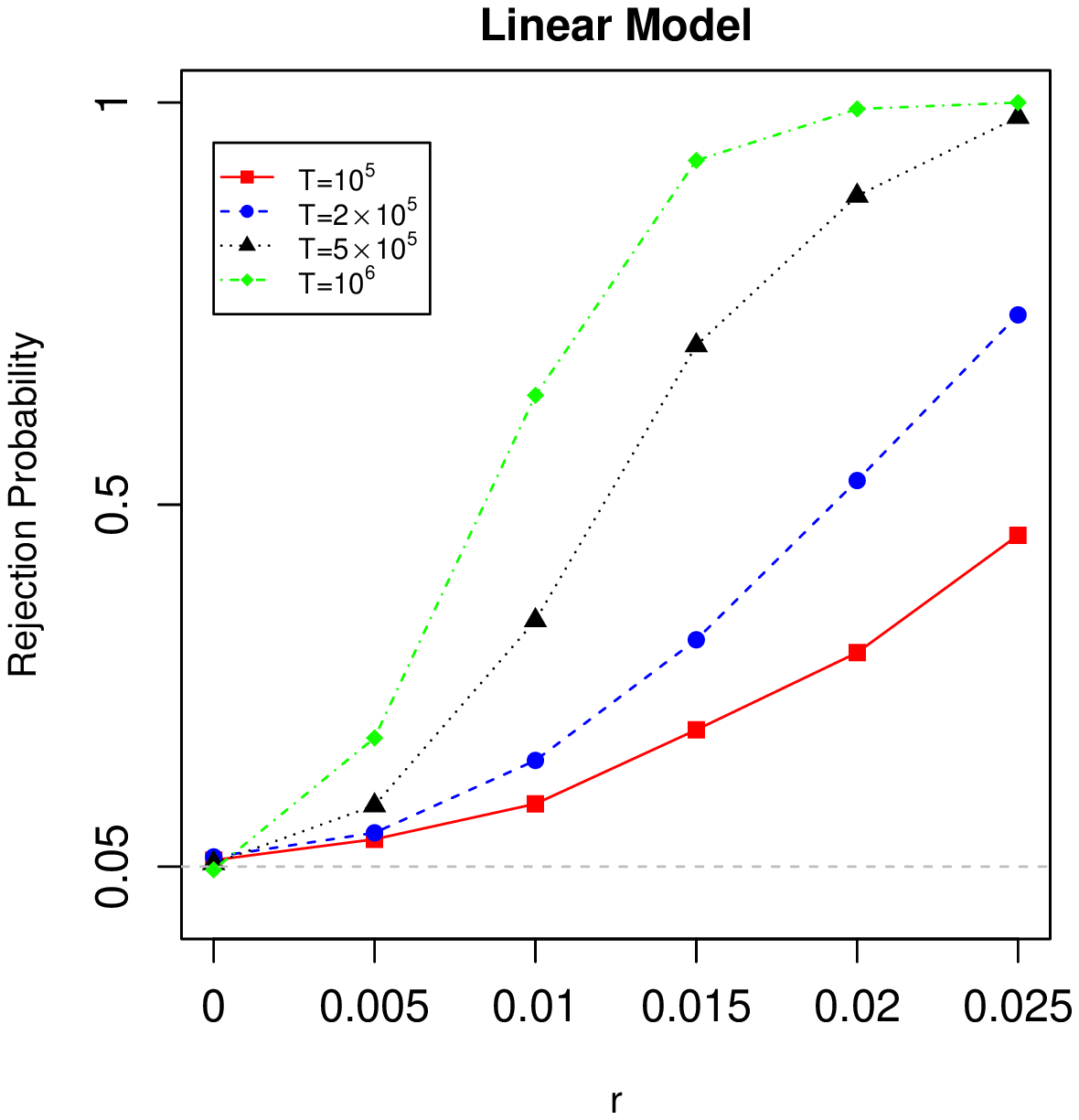}
\caption{Rejection probability of the specification test.}
\label{figure:test}
\end{figure}
\section{Empirical application}\label{section:real:data}
In this section, we apply our method to the protein tertiary structure dataset from UCI machine learning repository (\cite{Dua:2019}). The dataset contains response variable \textit{size of the residue} and nine explanatory variables (denoted by V1-V9) measuring the physicochemical properties of the  protein tertiary structure. There are 45730 observations in the dataset. After standardizing the explanatory variables, we first obtain the ASGD estimate (using $P=I$) and its standard error. We next calculate the corresponding p-values to examine whether the coefficients are significant, and the results are summarized in Table \ref{table:realdata:1}. Based on the ASGD estimate, the p-values of the explanatory variables V1, V5, V9 are $0.185, 0.179$ and $0.198$, which are not significant under significant level $\alpha=0.05$. Meanwhile, the variable  V7 has a p-value 0.069, which is close to 0.05. To determine whether removing these insignificant variables or not, we sequentially apply the testing procedure to test the following null hypotheses based on the p-value of the insignificant variables: V1=V5=V7=V9=0, V1=V5=V9=0, V1=V9=0, V9=0. The corresponding p-values of the specification test are $0, 0, 0.207$ and $0.198$. Based on the results, we fail to reject the null hypotheses V1=V9=0 and V9=0. Therefore, we calculate the APSGD estimate using the linear-equality constraint V1=V9=0, and the results are reported in Table \ref{table:realdata:2}. In comparison with the ASGD estimate, the APSGD estimate gives a smaller standard error for the estimated coefficients. Moreover, all the variables, except V7, are highly significant with p-value being almost zero. The p-values of V7 in APSGD and ASGD estimates are 0.055 and 0.069, respectively, which suggests its significance is slightly improved.

%\textit{total surface area, nonpolar exposed area, fractional area of exposed non polar residue, fractional area of exposed non polar part of residue, molecular mass weighted exposed area, average deviation from standard exposed area of residue, Euclidian distance, secondary structure penalty, spacial Distribution constraints.} 

%\begin{table}[hp!]
%\centering
%\caption{ASGD estimate. The symbols $^*$ and $^\bullet$ stand for p-value$<$0.05 and p-value$<$0.1.}
%\label{table:realdata:1}
%%\resizebox{1\textwidth}{!}{%
%\begin{tabular}{ccccccccccc}\hline\hline
%          & Intercept & V1    & V2    & V3    & V4     & V5    & V6     & V7    & V8    & V9     \\\hline
%PSGD & 7.782     & 1.187 & 3.094 & 0.901 & -4.713 & 1.019 & -2.303 & 1.159 & 0.884 & -0.231 \\
%SE        & 0.032     & 0.896 & 0.343 & 0.146 & 0.312  & 0.759 & 0.333  & 0.637 & 0.041 & 0.180  \\
%P-value   & 0.000$^*$     & 0.185 & 0.000$^*$ & 0.000$^*$ & 0.000$^*$  & 0.179 & 0.000$^*$  & 0.069$^\bullet$ & 0.000$^*$ & 0.198 \\\hline\hline
%\end{tabular}
%%}
%\end{table}

\begin{table}[t!]
\centering
\caption{ASGD estimate. The symbols $^*$ and $^\bullet$ stand for p-value$<$0.05 and p-value$<$0.1.}
\label{table:realdata:1}
%\resizebox{1\textwidth}{!}{%
\begin{tabular}{cccccccccc}\hline\hline
           & V1    & V2    & V3    & V4     & V5    & V6     & V7    & V8    & V9     \\\hline
PSGD      & 1.187 & 3.094 & 0.901 & -4.713 & 1.019 & -2.303 & 1.159 & 0.884 & -0.231 \\
SE            & 0.896 & 0.343 & 0.146 & 0.312  & 0.759 & 0.333  & 0.637 & 0.041 & 0.180  \\
P-value       & 0.185 & 0.000$^*$ & 0.000$^*$ & 0.000$^*$  & 0.179 & 0.000$^*$  & 0.069$^\bullet$ & 0.000$^*$ & 0.198 \\\hline\hline
\end{tabular}
%}
\end{table}
\begin{table}[t!]
  \centering
\caption{APSGD estimate with constraint V1=V9=0. The symbols $^*$ and $^\bullet$ stand for p-value$<$0.05 and p-value$<$0.1.}
\label{table:realdata:2}
\begin{tabular}{cccccccc}\hline\hline
           & V2    & V3    & V4     & V5    & V6     & V7    & V8    \\\hline
Estimates     & 3.468 & 0.671 & -4.971 & 2.086 & -1.681 & 0.466 & 0.906 \\
SE            & 0.208 & 0.085 & 0.149  & 0.216 & 0.188  & 0.243 & 0.034 \\
P-value        & 0.000$^*$ & 0.000$^*$ & 0.000$^*$  & 0.000$^*$ & 0.000$^*$  & 0.055$^\bullet$ & 0.000$^*$\\\hline\hline
\end{tabular}
\end{table}

\end{document}